\documentclass[12pt]{article}
\usepackage{amsmath,amsthm,amsfonts,amssymb,mathrsfs,bm}
\usepackage[usenames]{color}
\usepackage{hyperref}
\usepackage{amssymb}

\parskip 	\bigskipamount

\newtheorem{theorem}{Theorem}[section]
\newtheorem{lemma}[theorem]{Lemma}
\newtheorem{corollary}[theorem]{Corollary}
\newtheorem{proposition}[theorem]{Proposition}

\newtheorem{remark}{Remark}[section]

\newtheorem{definition}{Definition}

\usepackage{amsmath}
\usepackage{amssymb}
\usepackage{amsthm}
\usepackage{enumitem}
\usepackage{nicematrix}
\usepackage{booktabs}
\usepackage{mathrsfs}
\usepackage{natbib}

\renewcommand{\d}{\mathrm{d}}

\newcommand{\R}{\mathbb{R}}
\newcommand{\C}{\mathbb{C}}

\newcommand{\Z}{\mathbb{Z}}
\newcommand{\N}{\mathbb{N}}
\newcommand{\E}{\mathbb{E}}
\renewcommand{\P}{\mathbb{P}}

\newcommand{\Ccal}{\mathcal{C}}
\newcommand{\Scal}{\mathcal{S}}
\newcommand{\Xcal}{\mathcal{X}}
\newcommand{\Ycal}{\mathcal{Y}}
\newcommand{\Lcal}{\mathcal{L}}

\newcommand{\Fcal}{\mathcal{F}}
\newcommand{\Ical}{\mathcal{I}}

\newcommand{\var}{\mathrm{Var}}

\newcommand{\proj}{\mathrm{proj}}

\newcommand{\cont}{\mathrm{c}}
\newcommand{\disc}{\mathrm{d}}
\newcommand{\op}{\mathrm{op}}

\newcommand{\bfx}{\mathbf{x}}
\newcommand{\bfy}{\mathbf{y}}

\newcommand{\bfv}{\mathbf{v}}

\newcommand{\bfK}{\mathbf{K}}
\newcommand{\bfL}{\mathbf{L}}
\newcommand{\bfI}{\mathbf{I}}
\newcommand{\bfA}{\mathbf{A}}
\newcommand{\bfV}{\mathbf{V}}
\newcommand{\bfD}{\mathbf{D}}

\newcommand{\bfk}{\mathbf{k}}

\newcommand{\diag}{\mathrm{diag}}
\newcommand{\supp}{\mathrm{supp}}
\newcommand{\rank}{\mathrm{rank}}
\newcommand{\ndpp}{r\text{-}\mathrm{DPP}}
\newcommand{\Frob}{\mathrm{F}}

\renewcommand{\span}{\mathrm{span}}
\newcommand{\diam}{\mathrm{diam}}
\newcommand{\vol}{\mathrm{Vol}}

\renewcommand{\L}{\Lambda}

\makeatletter
\renewcommand*{\@cite@ofmt}{\hbox}
\makeatother

\begin{document}
\title{State-of-art minibatches via novel DPP kernels: discretization, wavelets, and rough objectives}
\author{
	\small\begin{tabular}{c}
		{Hoang-Son Tran \thanks{Co-first author}}\\  
        Department of Mathematics\\
		 National University of Singapore\\ transonsp97@gmail.com
	\end{tabular}
	\and
	\small\begin{tabular}{c}
		{Pranav Gupta \thanks{Co-first author}}
		 \\  Department of Mathematics\\
		 National University of Singapore\\
         pranav.gupta@u.nus.edu
	\end{tabular}
    \and
	\small\begin{tabular}{c}
		{R\'emi Bardenet}
		 \\  Univ. Lille, CNRS, Centrale Lille\\
		 UMR 9189 – CRIStAL\\
        remi.bardenet@cnrs.fr
	\end{tabular}
    \and
	\small\begin{tabular}{c}
		{Subhro Ghosh}
		 \\  Department of Mathematics\\
		 National University of Singapore\\
         subhrowork@gmail.com
	\end{tabular}
}
\date{}
\maketitle

\begin{abstract}
  Determinantal point processes (DPPs) have emerged as a kernelized alternative to vanilla independent sampling for generating efficient minibatches, coresets and other parsimonious representations of large-scale datasets. 
While theoretical foundations and promising  empirical performance have been demonstrated, there are two
challenges for current proposals for DPP-based coresets or minibatches. The first is the need for families of DPPs with certain key variance reduction properties, usually constructed in a continuous setting, of which there are few known examples. The second is the need for an ad-hoc construction of a discrete DPP defined on a given dataset, that inherits such variance reduction.
In this work, we contribute to the programme of establishing DPPs as a subsampling toolbox for ML by advancing on these two fronts.
First, we propose new DPPs on the Euclidean space based on wavelets, with provably better accuracy guarantees than the best known rates.
Second, we introduce a general method to convert such continuous DPPs, which are more amenable to proving analytical statements, into discrete kernels, which are pertinent for subsampling tasks such as minibatch and coreset constructions. 
This conversion mechanism simultaneously preserves the desired variance decay and reveals a low-rank decomposition of the discrete kernel, which makes sampling the corresponding DPP computationally inexpensive.
En route, we enlarge the class of ML tasks amenable to improvements via DPP-based minibatches and coresets to include objective functions with arbitrarily low regularity, and rate guarantees that explicitly adapt to this regularity.
\end{abstract}

 \tableofcontents

\newpage

\section{Introduction}
\label{sec:intro}
\emph{Determinantal point processes} (DPPs; \citep{Mac72}) are random point configurations that are parametrized by kernel functions.
Like kernel machines, DPPs are particularly tractable, both mathematically and computationally \citep{HKPV06,tran2025negative}. 
In machine learning, since the seminal work \cite{KuTa12}, they have been used as diversity-promoting models. 
Closer to our purpose, DPPs have also been shown to provably improve on independent sampling when designing numerical quadratures \citep{bardenet-aap,BeBaCh19,CoMaAm20Sub}, minibatches in stochastic gradient descent \citep{dpp-sgd} or building coresets \citep{TrBaAm19,coreset-dpp,jaquard2026}.

To fix ideas, for a configuration of points $\mathcal{S}\subset \mathbb{R}^d$ of cardinality $n$ and a class  $\mathcal{F}$ of functions, we define the \emph{linear statistic}
\begin{equation}
    \label{e:linear_statistic}
    \Lambda_{\mathcal{S}}(f) = \sum_{Y\in \mathcal{S}} f(Y),\quad f\in\mathcal{F}.
\end{equation}
The recipe for constructing a minibatch or coreset using DPPs typically requires two ingredients.
The first is a probabilistic guarantee that, when $\mathcal{S}=\mathcal{S}_{\mathrm{c}}$ is a specific DPP whose points are allowed to vary in all of $\mathbb{R}^d$ (loosely termed a ``continuous'' DPP), a suitable set of unbiased linear statistics \eqref{e:linear_statistic} has a standard error that decays faster than the classical rate $1/\sqrt{n}$ of i.i.d. sampling.
Yet both coresets and minibatches are constrained to be subsets of an $N$-point dataset $\mathcal{X}\subset\mathbb{R}^d$, so the second ingredient is to design a second DPP $\mathcal{S}_\mathrm{d}$ that is \emph{discrete}, i.e. supported on the subsets of $\mathcal{X}$, and inherits the variance decay of the original continuous DPP $\mathcal{S}_\mathrm{c}$.

\paragraph{Motivation and our contributions.}
Available designs for DPP minibatches and coresets with qualitatively improved theoretical performance over independent sampling \cite{dpp-sgd,TrBaAm19,coreset-dpp,jaquard2026} rely on DPPs known as multivariate orthogonal polynomial ensembles \citep{bardenet-aap}.
The latter come with a standard error for linear statistics of order $n^{-(\frac{1}{2}+\frac{1}{2d})}$, holding when the test functions in $\mathcal{F}$ are essentially $C^1$. 
A key contribution of this paper is to introduce a novel class of DPP samplers, defined using \emph{wavelets}, such that the standard error of $\Lambda(f)$ decays fast for quite general test functions, namely for Hölder continuous functions and functions belonging to fractional Sobolev spaces. 
Furthermore, the decay rates \emph{adapt to} the \emph{regularity of $f$} in an explicit manner, and in the $C^1$ setting provide an improved guarantee of a standard error in $n^{-(\frac12 + \frac1d)}$, which matches the best known rates in that setting \cite{LeBa24Sub} without complex-analytic machinery.
Our construction also \emph{improves} on the best rates obtained with \emph{weak smoothness requirements} for variance reduction with DPPs \citep{CoMaAm20Sub}, and opens up applications to e.g. coresets for \emph{non-smooth loss functions}, or quadrature in rough settings such as for financial data. 
An interesting by-product of our analysis is the identification of classical \emph{stratified sampling} in the continuum as a special case of a wavelet-based DPP sampler, thereby explaining its superior empirical performance as a coreset sampling mechanism \cite{coreset-dpp}, and providing natural motivation to pursue determinantal samplers as a robust and structured subset generation mechanism. 


The transfer of the favorable guarantees from the continuum background space to the discrete dataset has historically had the problem of being mathematically and computationally involved: the construction in \cite{dpp-sgd}, for instance, requires a cumbersome spectral truncation of an $N\times N$ matrix that is near-to but not low-rank. 
Alternatives include mimicking the continuum construction in the discrete setting, such as with the discrete OPE in \citep{TrBaAm19}, which requires customized arguments for their rates \cite{jaquard2026} and is very much tied to the structured scenario of orthogonal polynomials. 
A key contribution of our paper is a comprehensive pipeline for the \emph{conversion of continuum kernels} with salutary variance decay properties to \emph{canonical discrete kernels}, defined on the dataset, that inherit similar guarantees. The proposed procedure is both an algorithmic recipe that is \emph{computationally cheaper} than previous alternatives, as well as a mathematical technique that provably ensures the \emph{transfer of variance decay rates} from the continuum to the discrete. 
Using this pipeline as a turn-key mechanism potentially enables future investigations in this area to focus squarely on the more analytically tractable question of designing continuum DPP kernels with favourable variance decay. Our kernel transfer pipeline comes with an \emph{in-built low-rank decomposition} of the discrete kernel, which has important computational benefits in the context of modern DPP sampling algorithms \cite{KuTa12,gillenwater19a,barthelme2023faster}. 
The pipeline is based on certain novel stochastic identities that could be of independent probabilistic interest.   
Finally, our pipeline is complementary to, but different from, the recent general results of \cite{jaquard2026} on variance transfer from continuous to discrete DPPs: while they implicitly assume the availability of a natural candidate discrete DPP (like a discrete OPE) that can be shown in a case-by-case analysis to approximate the underlying continuous kernel in a suitable sense, we rather assume that we can evaluate the kernel of the continuous DPP and use it to \emph{define} a discrete DPP that automatically inherits its variance-decay properties. 

The rest of the paper is organized as follows.
In Section~\ref{sec:background}, we introduce some necessary background on DPPs.
In Section~\ref{sec:discreteDPP}, we present our novel discrete DPP construction; proofs are deferred to Appendix~\ref{a:appendix}.
In Section~\ref{sec:waveletDPP}, we present our new wavelet-based continuous DPPs; proofs are deferred to Appendix~\ref{a:appendix-wavelet}.
In Section~\ref{sec:applications}, we give two examples of how to put our two contributions together and obtain new theoretical guarantees for coresets and numerical integration; proofs are deferred to Appendix~\ref{a:appendix-application}.
In Section~\ref{sec:experiments}, we numerically validate our theoretical results; additional experiments can be found in Appendix~\ref{a:additional_experiments}.
We conclude with opening remarks in Section~\ref{sec:discussion}.

\section{Background} \label{sec:background}





{\bf DPPs.} 
Let $E$ be a Polish space equipped with a reference measure $\mu$. 
Typical examples of $E$ include Euclidean spaces, manifolds, and discrete spaces such as graphs. A \emph{point process} $\Scal$ on $E$ is a random locally finite subset of $E$. We say that $\Scal$ is a DPP (w.r.t. $\mu$) if there exists a measurable function $K: E^2\rightarrow \C$ such that \begin{equation}
    \label{eq:definition_DPPs}
    \E \Big [ \sum_{\neq} f(x_{i_1},\ldots,x_{i_k}) \Big] = \int_{E^k} f(x_1,\ldots,x_k) \det [K(x_i,x_j)]_{k\times k} \, \d \mu^{\otimes k} (x_1,\ldots, x_k),
\end{equation}
where the sum in the LHS ranges over all pairwise distinct $k$-tuples of $\Scal$, for all bounded measurable $f:E^k \rightarrow \R$ and for all $k \in \N$. Such a function $K$ is called a \emph{kernel} of the DPP $\Scal$, and $\mu$ is the \emph{reference measure} (see \citep{Mac75, HKPV06}). We will denote $\mathcal{S}\sim \mathrm{DPP}(K,\mu)$.

An important class of DPPs is the \emph{projection DPPs}, which are defined by projection kernels. To elaborate, let $n \ge 1$ and $\phi_1, \dots, \phi_n$ be $n$ (representatives of) orthonormal functions in $L^2(\mu)$, and define $K(x,y):= \sum_{k=1}^n \phi_k(x)\phi_k(y)$.  Notice that, $K$, as an integral kernel, defines a projection operator on the space $L^2(\mu)$. The resultant DPP defined by $(K,\mu)$ is called a \emph{projection DPP} of rank $n$. A useful probabilistic interpretation of projection DPPs is the following: if $\mathcal{S}\sim \mathrm{DPP}(K,\mu)$ where $K$ is a projection kernel of rank $n$, then 
\(\Scal=\{X_1,\ldots,X_n\}\) can be obtained by drawing \((X_1,\ldots,X_n)\) from the probability measure
\begin{equation}
    \label{e:density_dpp}
    p(x_1,\dots,x_n)\,\d\mu(x_1)\cdots \d\mu(x_n)
    :=
    \frac{1}{n!}
    \det\!\bigl(K(x_i,x_j)\bigr)_{1\leq i,j\leq n}
    \,\d\mu(x_1)\cdots \d\mu(x_n).
\end{equation}
For more discussions on DPPs, we refer to the surveys \cite{HKPV06,KuTa12}.


{\bf Discrete DPPs.} 
When $E$ has finite cardinality $N$ and $\mu$ is taken to be counting measure on $E$, we speak of DPP($K$, $\mu$) as a \emph{discrete DPP on $E$}. 
One can check that this is equivalent to the following more convenient definition. A random subset $\Scal\subset E$ is a discrete DPP if there is an $N\times N$-matrix $\bfK$ indexed by elements of $E$ such that
\( \P(\Scal \supset T) = \det [\bfK_{T}]\) for all subset $T\subset E$.
Here $\bfK_T$ denotes the principal submatrix of $\bfK$ whose rows and columns are indexed by the elements of $T$. 

A popular class of discrete DPPs in ML is given by
\emph{\(L\)-ensembles}, which are random subsets \(\mathcal{S}\subset E\) satisfying
\(
    \P(\Scal=T) \propto \det(\bfL_T)\)
for all subsets \(T\subset E\),
where \(\bfL\) is an \(N\times N\) symmetric matrix indexed by \(E\), with positive eigenvalues.
It is known that every \(L\)-ensemble is a DPP with kernel $\bfK=\bfL (\bfI+\bfL)^{-1}$. However, not every DPP is an \(L\)-ensemble. In particular, an \(L\)-ensemble has random cardinality, whereas a projection DPP has deterministic cardinality. To obtain a subset of fixed size when sampling with $L$-ensembles, a standard approach is to condition an \(L\)-ensemble $\Scal$ on the event \(|\Scal|=m\), for some \(m\in\N\) with $m\leq \rank(\bfL)$. This gives rise to the so-called \emph{\(m\)-DPPs}; see \cite{KuTa12}.
We also refer the interested reader to \cite{TBUA23} for a unifying concept of \emph{extended} $L$-ensembles.

{\bf Orthogonal Polynomial Ensembles (OPEs).}
When $E = \mathbb{R}^d$ and $\mu$ has a density w.r.t. Lebesgue, we speak of DPP($K, \mu$) as a \emph{continuous} DPP. 
A classical example consists in letting $\phi_1, \dots, \phi_n$ be the result of applying the Gram-Schmidt orthonormalization process in $L^2(\mu)$ to the $n$ first multivariate monomials, after having chosen an order on multivariate indices. 
The resulting DPP is called \emph{multivariate orthogonal polynomial ensemble} (OPE), and we denote its kernel as $K_{\mathrm{OPE}, \mu}$.
Under conditions on $\mu$, when $\mathcal{S}$ is the multivariate OPE, for smooth enough $f:\mathbb{R}^d \rightarrow \mathbb{R}$, and the weight function $w : x \mapsto K_{\mathrm{OPE},\mu}(x,x)^{-1}$, \cite{bardenet-aap} showed that the linear statistic $\Lambda_{\Scal}(fw)$ satisfies $\mathbb{E} \left[\Lambda_{\Scal}(fw)\right] = \int f\d\mu$ and $\sqrt{\var \left[\Lambda_{\Scal}(fw)\right]} = O(1/n^{\frac{1}{2}+\frac{1}{2d}})$, improving over the classical Monte Carlo rate of $O(1/\sqrt{n})$. This was leveraged in \cite{dpp-sgd} to construct DPP-based minibatch samplers, and in \cite{coreset-dpp} to devise coreset sampling mechanisms.

\section{A general pipeline from the continuum to the discrete} 
\label{sec:discreteDPP}
In line with Section~\ref{sec:background}, let $E$ be a Polish space equipped with a reference $\mu$, and $\mathcal{S}_\mathrm{c}\sim\mathrm{DPP}(K,\mu)$ with $K$ the projection kernel onto the span of $n$ orthonormal functions $\psi_1, \dots, \psi_n \in L^2(\mu)$. 
Let $X_1, \dots, X_N\sim \nu$ i.i.d., where $\nu(\d x) = \rho(x)\mu(\d x)$ is a probability measure with $\rho(x)>0$ for $\mu$-a.e $x$. 
We assume that $\nu$ is non-atomic so that the samples $X_1, \dots, X_N$ are pairwise distinct, almost surely.
We think of $\mathcal{X} = \{X_1, \dots, X_N\}$ as a dataset. 
Conditionally on $\mathcal{X}$, in this section we build two discrete DPPs that inherit the variance of linear statistics of $\mathcal{S}_\mathrm{c}$.



First, define the matrix
\begin{equation}
    \label{e:Phi}
    \mathbf{\Psi}:= \begin{pmatrix}
        \psi_1(X_1) & \ldots & \psi_n(X_1) \\
        \vdots &\ddots &\vdots \\
       \psi_1(X_N)  & \ldots & \psi_n(X_N)
    \end{pmatrix} = [\psi_k(X_i)] \in \C^{N\times n},
\end{equation}
and set $m:= \rank(\mathbf{\Psi})$. We remark that $\mathbf{\Psi}$ and $m$ are measurable with respect to the data set $\Xcal$.

\begin{remark}
We note that $m$ is random in general, and clearly $m\le n$ almost surely. However, we have
$\mathbb{E}[\mathbf{\Psi}^*\bfD(\rho^{-1})\mathbf{\Psi}] = N\mathbf{I}_n$,
where $\bfD(\rho^{-1})=\diag(\rho(X_1)^{-1},\ldots,\rho(X_N)^{-1})$. This suggests that $\mathbf{\Psi}$ has maximal rank $m=n$ with high probability. In many examples of interest, such as graded polynomial feature functions $\psi_i$ with $\nu$ having a density w.r.t. Lebesgue, we even have $m=n$ with probability $1$. 
\end{remark}

Now, conditionally on $\mathcal{X}$, define $\mathcal{S}_\mathrm{d}$ as the $m$-DPP on $\Xcal$ corresponding to the $L$-ensemble defined by
\begin{equation}
    \label{e:L}
    \bfL = \frac{1}{N}\bfD(\rho^{-1/2})\,\mathbf{\Psi}\mathbf{\Psi}^*\,\bfD(\rho^{-1/2}),
\end{equation} 
see Section~\ref{sec:background} for definitions.
Alternately, if we do not have query access to $\rho$ but only to a (positive) estimator $\hat\rho$ of $\rho$, we shall consider the $m$-DPP $\widehat{\Scal}_{\disc}$ on $\mathcal{X}$ corresponding to the $L$-ensemble defined by
\begin{equation}
    \label{e:LHat}
    \widehat{\bfL} = \frac{1}{N}\bfD(\hat\rho^{-1/2})\,\mathbf{\Psi}\mathbf{\Psi}^*\,\bfD(\hat\rho^{-1/2}).
\end{equation} 


We prove the following in the appendix, see Lemma \ref{lm:n-dpp}.
\begin{proposition}
Both $\Scal_{\disc}$ and $\widehat{\Scal}_{\disc}$ are actually (discrete) projection DPPs on $\Xcal$, with respective kernel the projection onto the span of the eigenvectors of $\mathbf{L}$ and $\widehat{\mathbf{L}}$.
\end{proposition}
In particular, there exist algorithms to exactly sample $\Scal_{\disc}$ or $\widehat{\Scal}_{\disc}$, at a cost dominated by the computation of the eigenvectors of $\mathbf{L}$ and $\widehat{\mathbf{L}}$, respectively; see \cite{barthelme2023faster} and references therein. 
Rewriting e.g. \eqref{e:L} as $\mathbf{B}\mathbf{B}^*$, one realizes that the cost of finding these eigenvectors by rather diagonalizing $\mathbf{B}^*\mathbf{B}$ is $O(Nn^2+n^3)$.
In minibatch and coreset constructions, where $n\ll N$, the sampling cost is thus essentially linear in $N$.
Finally, note that unlike the discrete OPEs used in \cite{TrBaAm19,jaquard2026}, the definition of our DPP kernels requires the evaluation of the eigenfunctions of the kernel $K$ of $\mathcal{S}_{\mathrm{c}}$ in \eqref{e:Phi}. Note that the kernels used for subsampling tasks with DPPs are typically projection kernels that have accessible Mercer decompositions (i.e., eigenfunction expansions), so in practice it is not even necessary to compute these eigenfunctions separately.
We now show that $\mathcal{S}_\mathrm{d}$ and $\widehat{\mathcal{S}_\mathrm{d}}$ inherit the variance reduction of linear statistics of $\mathcal{S}_\mathrm{c}$, without any argument specific to the choice of the kernel $K$. 
This leaves the choice of the kernel as a flexibility in the hands of the user.

\begin{theorem}[Discretization Pipeline - I] \label{thm:disc-cont-true}
Let $f:E \rightarrow \R$ be a bounded test function. 
Assume that
 \begin{equation} \label{eq:Kmax}
    \exists\: K_{\max}<\infty \text{ such that } \frac{K(x,x)}{\rho(x)} \le K_{\max} \: \text{ for $\nu$-almost every $x$} \tag{A1}. 
\end{equation}
Recall the definition of linear statistics in \eqref{e:linear_statistic}.
Then for any $\delta,\delta' \in (0,1)$, with probability at least $1-\delta-\delta'$ in the data set $\Xcal$, we have
\begin{eqnarray*}
    \var[\L_{\Scal_{\disc}}(f)|\Xcal] &\le&  2\var[\L_{\Scal_{\cont}}(f)] +4\|f\|_{\infty}^2K_{\max}^2\mathscr{E}(n,N,\delta,\delta'),  \quad \text{and}\\
    \var[\L_{\Scal_{\disc}}(f)|\Xcal] &\ge&  \frac{N-1}{2N}\var[\L_{\Scal_{\cont}}(f)] - 2\|f\|_{\infty}^2K_{\max}^2\mathscr{E}(n,N,\delta,\delta'),
\end{eqnarray*}
where
    \begin{equation} \label{eq:def-Error}
     \mathscr{E}(n,N,\delta,\delta'):= 4\sqrt{\frac{2\log (2/\delta')}{N}} + \frac{4}{9N^2} \log^2 \Big (\frac{n^2+1}{\delta}\Big )+ \frac{4}{N} \log \Big (\frac{n^2+1}{\delta}\Big ) .
\end{equation}
\end{theorem}

Several comments are in order. 
First, in the cases we have in mind, $K_{\max}\asymp n$. 
This is true under mild assumptions for multivariate OPEs, see Section~\ref{s:more_remarks_on_the_theorem}, and for the wavelet-based DPPs we will introduce in Section~\ref{sec:waveletDPP}.
Second, our algorithm produces a minibatch \(\Scal_{\disc}\) of size \(m \le n\), rather than one that is necessarily of size exactly \(n\). However, we note that if \(N\) is sufficiently large so that $K^2_{\max} \mathscr{E}(n,N,\delta,\delta') <1$ (which is the relevant regime in practice), then on the same event appearing in the conclusion of Theorem~\ref{thm:disc-cont-true}, the minibatch has size exactly \(n\). This follows from our proof technique; see Lemma~\ref{lm:bernstein}.
Third, and most importantly, Theorem \ref{thm:disc-cont-true} establishes a \emph{variance-transfer guarantee}: the conditional variance of the linear statistic associated with the subsample $\mathcal{S}_\mathrm{d}$ is bounded above and below by that of the underlying continuous DPP, up to \emph{explicit} error terms. In the regime most relevant for applications, where $N\gg n$, these error terms are negligible, leaving \(\var[\L_{\Scal_{\cont}}(f)]\) as the dominant term.
By choosing a continuous DPP kernel $K$ that is known to enjoy variance reduction (i.e., \(\var[\L_{\Scal_{\cont}}(f)]=O(n^{1-\alpha})\) for some \(\alpha>0\)), the resulting minibatch $\Scal_{\disc}$ thus automatically inherits this variance reduction property. 

In the regime where only an estimator $\hat{\rho}$ is available, we show that $\widehat{\Scal}_{\disc}$ defined by \eqref{e:LHat} enjoys the same bounds as $\Scal_{\disc}$, up to an explicit additional error coming from density estimation.

\begin{theorem} [Discretization Pipeline - II] \label{thm:disc-cont-approx}
    Let $f:E \rightarrow \R$ be a bounded test function. 
 Assume both \eqref{eq:Kmax} and that
 \begin{equation} \label{eq:rhohat}
      \P\Big (\max_{1\le i \le N}\Big |\frac{\rho(X_i)}{\hat\rho(X_i)} -1 \Big | > \varepsilon \Big ) \le \beta, \quad \text{for some }\varepsilon \in [0,1), \beta \in [0,1). \tag{A2}
    \end{equation}
 
 Then, for any $\delta,\delta' \in (0,1)$, with probability at least $1-\delta-\delta'-\beta$ in the data set $\Xcal$,
  \begin{eqnarray*}
      \var[\L_{\widehat{\Scal}_{\disc}}(f)|\Xcal]&\le& 2(1+\varepsilon)^2 \var[\L_{\Scal_{\cont}}(f)] + 32\|f\|^2_{\infty} K_{\max}^2\mathscr{E}(n,N,\delta,\delta') + 8 \|f\|^2_{\infty}\varepsilon^2n, \quad \text{and} \\
      \var[\L_{\widehat{\Scal}_{\disc}}(f)|\Xcal] &\ge& \frac{(N-1)}{2N}(1-\varepsilon)^2\var[\L_{\Scal_{\cont}}(f)] - 16 \|f\|^2_{\infty} K_{\max}^2\mathscr{E}(n,N,\delta,\delta') - 4 \|f\|^2_{\infty}\varepsilon^2n.
  \end{eqnarray*}
\end{theorem}


Compared to Theorem~\ref{thm:disc-cont-approx}, $\widehat{\Scal}_{\disc}$ only suffers an additional error term of order $O(\varepsilon^2 n)$. This term is not restrictive for our purposes. Indeed, to retain the variance reduction property, it is enough that $\varepsilon = O(n^{-\alpha})$ for some $\alpha>0$. 
In practice, one may even obtain estimators for which $\varepsilon$ decays with the full sample size $N$, which is much larger than the minibatch size $n$. For instance, in \cite{kde-jiang-icml}, if the density $\rho$ is $\alpha$-H\"older continuous on $\R^d$, then the KDE estimator $\hat\rho_{\mathrm{KDE}}$ satisfies $\|\hat\rho_{\mathrm{KDE}} - \rho\|_{\infty} = O(N^{-\alpha/(2\alpha+d)})$ (ignoring logarithmic factors) with high probability (c.f. Theorem 2 and Remark 8 in \cite{kde-jiang-icml}). 
That being said, we point out that we actually only need in \eqref{eq:rhohat} to control the relative error of the density estimate \emph{on the dataset} $\{X_i\}_{i=1}^N$, with high probability.
This should allow us to go beyond the usual $L^\infty$ bounds on $\hat{\rho}-\rho$ and work in much weaker settings.


\section{A new class of DPPs: wavelet bases and the variance reduction phenomenon} \label{sec:waveletDPP}
In this section, we introduce a new class of DPPs constructed from wavelets and show that they come with fast-decaying variance of linear statistics even for non-smooth test functions, making them ideal candidates as the continuous DPP in Theorems~\ref{thm:disc-cont-true} and \ref{thm:disc-cont-approx}. 
To the best of our knowledge, this yields the first example of a class of DPPs in arbitrary dimension with this property that does not arise from orthogonal polynomial ensembles \citep{bardenet-aap} or the Dirichlet ensemble \cite{CoMaAm20Sub}.

\paragraph{From wavelets to a DPP kernel.} Let $\phi:\R \rightarrow \R$ be a bounded, compactly supported, orthonormal scaling function, associated with a multiresolution analysis, normalized by $\int\phi(x)\d x=1$.
We refer to Appendix~\ref{a:appendix-wavelet} for a self-contained introduction to multiresolution analysis.
A typical example would be the Haar scaling function $\phi_{\mathrm{Haar}}(x) = \mathbf{1}_{[0,1)}(x)$.
For each $j,k \in \Z$, we define
\[\phi_{j,k}(x):=2^{-j/2} \phi (2^{-j}x-k)\quad,\quad \Phi_{j,\bfk}(\bfx) = \prod_{i=1}^d \phi_{j,k_i}(x_i),  \]
where $\bfk=(k_1,\ldots,k_d)\in\Z^d$ and $\bfx = (x_1,\ldots,x_d)\in \R^d.$
Fix $j\in \N$, we define an index set
\[ \Ical=\Ical(j):= \{ \bfk=(k_1,\ldots,k_d) \in \Z^d : \supp(\Phi_{-j,\bfk})\subset [0,1]^d \}.\]
We define the wavelet DPP associated to the scaling function $\phi$ as the DPP $\mathcal{S}_\mathrm{c}$ with reference measure the Lebesgue measure on $\mathbb{R}^d$, and kernel
\begin{equation} \label{eq:wavelet-dpp-kernel}
    K(\bfx,\bfy) := \sum_{\bfk\in \Ical} \Phi_{-j,\bfk}(\bfx)\Phi_{-j,\bfk}(\bfy), \quad \bfx,\bfy \in \R^d.
\end{equation}

We highlight that $j$ is fixed, and that $K$ is parametrized by the scaling function $\phi$ and the integer $j$. 
Moreover, by construction, the $(\Phi_{-j,\bfk})_\bfk$ are orthonormal in $L^2(\R^d, \d \bfx)$, so that $K$ is a projection kernel of rank $n:= |\Ical(j)|$. 
We thus control the cardinality $n$ of our DPP through the choice of $j$ and $\phi$.
In particular, a simple exercise shows that if $\supp(\phi)\subset [a,b]$ then $n = (\lfloor 2^j-b\rfloor - \lceil -a \rceil +1)^d \asymp 2^{dj}$. 
While the joint choice of $j, a, b$ can help choose a particular cardinality, asymptotic results will let $j\rightarrow\infty$, and as such the number of points will be exponential in the dimension. 
This is a hint that wavelet DPPs should be used after some low-dimensional approximation, e.g. after running a PCA, as we demonstrate in the experimental Section~\ref{sec:experiments}. On the other hand, when $d$ is the effective dimension of the data, an $\exp(d)$ points would be minimally required to fully ``explore the space'', so this sample size is rather reasonable.
As we shall now see, the DPP with kernel \eqref{eq:wavelet-dpp-kernel} generalizes stratified sampling.

\begin{proposition}[Connections to Stratified Sampling]    \label{p:stratified_sampling}
    When the scaling function is $\phi_{\mathrm{Haar}}(x) = \mathbf{1}_{[0,1)}(x)$, we have $n =2^{dj}$.
    Furthermore, the supports of $\Phi_{-j,\bfk}$, $\bfk \in \Ical$, are disjoint, and $\mathcal{S}_\mathrm{c}\sim \mathrm{DPP}(K,\d \bfx)$ has the same distribution as the set $\{Y_\bfk: \bfk \in \Ical \}$, where the $Y_\bfk$s are independent and $Y_\bfk$ is drawn uniformly on $\mathrm{supp} (\Phi_{-j,\bfk})$. 
\end{proposition}

 Our continuous-to-discrete pipeline (see below) subsequently enables us to transfer this to a new kind of discretized analogue of stratified sampling that is tuned to the problem of subset sampling from a large dataset.


\paragraph{Variance reduction.}
Consider bounded, compactly supported test functions, 
\begin{equation}\label{eq:cpt-supp} 
    f: \R^d \rightarrow \R , f(\bfx) = 0 \text{ on } \R^d \setminus (0,1)^d \text{ and } \|f\|_{\infty}:= \sup_{\bfx\in \R^d} f(\bfx) < \infty. \tag{$*$}
\end{equation}
To study the impact of smoothness, we further consider two classes of functions: the space $C^{0,s}$ of Hölder continuous functions of order $s$, and the Sobolev space $H^s$, both with parameter $s\in (0,1]$. For completeness, we recall the definitions of these spaces and their norms in Appendix \ref{sec:test-function}.
We denote by $\mathcal{F}_s$ the space of functions $f$ satisfying \eqref{eq:cpt-supp} and belonging to either $C^{0,s}$ or $H^s$. 
Furthermore, when clear from context what space $f$ belongs to, we denote by $|f|_s$: either (a) for H\"older classes, the H\"older seminorm $|f|_{C^{0,s}}$  ($s \in (0,1]$); or (b) for Sobolev classes, the Gagliardo seminorm $|f|_{H^s}$ (for $s \in (0,1)$) and $\|\nabla f\|_{L^2}$ (for $s=1$).

\begin{theorem}[Variance reduction for wavelet kernels] \label{thm:wavelet-dpp}
Let $s\in (0,1]$. Consider the DPP $\Scal_\mathrm{c}$ on $(\R^d, \d\bfx)$ defined by the kernel \eqref{eq:wavelet-dpp-kernel}, then there exists a constant $C$ (depending only on $\phi$ and $d$) such that, 
$$ 
    \forall f\in\mathcal{F}_s, \quad
\var[\L_{\Scal}(f)] \le  C \,|f|^2_{s} \, n^{1-2s/d}.
$$
\end{theorem}

It is remarkable that when $s=1$, we have a variance of the order $n^{1-2/d}$, compared to $n^{1-1/d}$ for multivariate OPEs.
We thus match the rate obtained in \cite{LeBa24Sub}, without the strong assumption that we are working on a compact complex manifold. For the Sobolev space $H^s$, our construction improves on the rate $n^{1 - \min(2s, 1)/d}$ obtained in \cite{CoMaAm20Sub} when $s > 1/2$, and matches it when $s \leq 1/2$. 
In the special case $\phi=\phi_{\mathrm{Haar}}$, we note that Theorem~\ref{thm:wavelet-dpp} provides explicit {\it variance reduction guarantees for stratified sampling} for objective functions of {\it low regularity} ($s<1$), which is potentially new and of independent interest.

\section{Applications: quadrature and coreset construction} \label{sec:applications}
Linear statistics \eqref{e:linear_statistic} abound in ML, so that a variety of results can be derived from the results of Sections~\ref{sec:discreteDPP} and \ref{sec:waveletDPP}.
As examples, we show first how to derive faster-than-Monte-Carlo quadrature from Theorem~\ref{sec:waveletDPP}, and then how to combine this fast quadrature with the continuous-to-discrete conversion in Section~\ref{sec:discreteDPP} to obtain more-accurate-than-independent coresets. 

\paragraph{Quadrature using wavelet DPPs.}

Let $\omega:\mathbb{R}^d\to\mathbb{R}$ be a $C^1$ density with\footnote{Hereafter, $A\Subset B$ means that the closure of $A$ is a compact subset of $B$.} $\supp(\omega)\Subset (0,1)^d$.
Consider approximating
\(
    \int f(\bfx)\omega(\bfx)\d\bfx,
\)
where $f\in\Fcal$ belongs to a prescribed class of integrable functions.
Let $K$ be the wavelet DPP kernel \eqref{eq:wavelet-dpp-kernel} for some $j$ and $\phi$, and $\mathcal{S}_\mathrm{c} = \{Y_1,, \dots, Y_n\}$ be the corresponding DPP.
For $f:\mathbb{R}^d\rightarrow\mathbb{R}$, it is natural to define 
\begin{equation}
    \label{e:first_MC_estimator}
    \Lcal_{\Scal_\mathrm{c}}(f) := \sum_{i=1}^n \frac{f(Y_i) \omega(Y_i)}{K(Y_i,Y_i)},
\end{equation}
noting that this is a linear statistic \eqref{e:linear_statistic} of the DPP $\Scal_\mathrm{c}$.
In particular, one can deduce from \eqref{e:density_dpp} that \eqref{e:first_MC_estimator} is an unbiased estimator of $\int f\omega \d x$.

In the special case $\phi=\phi_{\mathrm{Haar}}$, the kernel diagonal $x\mapsto K(x,x)$ is constant, and we prove in Proposition~\ref{p:wavelet-dpp-nor} that the mean squared error of \eqref{e:first_MC_estimator} scales as $O(n^{-1-2s/d})$ for $f$ in either $C^{0,s}$ or $H^s$.
In view of Proposition~\ref{p:stratified_sampling}, this was somehow expected: indeed, we know since the early work of \citep{haber1966modified} that the MSE of the stratified sampling estimator is $O(n^{-1-2/d})$ for $f: [0,1]^d\rightarrow \mathbb{R}$ continuously differentiable.
As noted in \cite{CoMaAm20Sub}, this rate outperforms the OPE rate $O(n^{-1-1/d})$ for the same kind of estimator in \cite{bardenet-aap}, though the choice of the cardinality of the Haar-wavelet-based quadrature is limited to the set $\{2^{d}, 2^{2d}, \dots\}$, making it impractical in large dimensions. 

For a general scaling function, the choice of the cardinality is more flexible, but the diagonal $x\mapsto K(x,x)$ is not necessarily constant, and we have found it more mathematically tractable to modify the linear statistic \eqref{e:first_MC_estimator} as follows.
With the notation of Section~\ref{sec:waveletDPP}, fix a design \(\{\bfx_{\bfk}:\bfk\in\Ical\}\) such that, for each \(\bfk\in\Ical\), \(\bfx_{\bfk}\) is any fixed point in \(\supp(\Phi_{-j,\bfk})\).
We consider the estimator
\begin{equation}\label{eq:adjusted-linear-stat}
    \widetilde{\Lcal}_{\Scal_\mathrm{c}}(f) 
    :=  \Lcal_{\Scal_\mathrm{c}}(f) 
    - \frac{1}{2^{dj/2}}\sum_{i=1}^n \sum_{\bfk \in \Ical} \frac{\Phi_{-j,\bfk}(Y_i)}{K(Y_i,Y_i)}  f(\bfx_\bfk)\omega(\bfx_\bfk)
    + \frac{1}{2^{dj}}\sum_{\bfk \in \Ical} f(\bfx_\bfk)\omega(\bfx_\bfk).
\end{equation}
This new linear statistic is reminiscent of control variates estimators in Monte Carlo integration \cite[Section 4.4.2]{RoCa04}.
At the price of a moderate increase of $n$ to $2n$ in the number of required evaluations of $f$ when changing from \eqref{e:first_MC_estimator} to \eqref{eq:adjusted-linear-stat}, our main quadrature result states that the MSE of the latter scales like that of the former, without any restriction on $\phi$.

\begin{theorem}[Fast quadrature for Hölder functions]
    \label{thm:wavelet-dpp-adj}
    Let $\Scal_\mathrm{c}$ be  the DPP on $(\R^d, \d\bfx)$ defined by the kernel $K$ in \eqref{eq:wavelet-dpp-kernel}. 
    Let $s\in(0,1]$.
    Then for $f \in C^{0,s}$ satisfying \eqref{eq:cpt-supp},
    \[\E \Big [ \Big |\widetilde{\Lcal}_{\Scal}(f) - \int f(\bfx)\omega(\bfx) \d \bfx) \Big |^2 \Big ]  = O(n^{-1-2s/d}).\]
\end{theorem}

\paragraph{Discretized Haar wavelet DPPs as coresets.}
Let \(\Xcal\) be a data set of \(N\) points, and let \(\Fcal\) be a family of test functions on \(\Xcal\). In ML, we think of $\Fcal$ as indexed by predictors, and evaluating $f$ at $x$ gives the loss incurred by the predictor indexed by $f$ at point $x$.
For each \(f\in\Fcal\), define
\(
    L(f) := \sum_{x\in\Xcal} f(x).
\)
When $N$ is huge, it is natural to approximate \(L(f)\) by a weighted sum of the form
\(
    L_{\Ccal}(f) := \sum_{x\in\Ccal} w(x) f(x),
\)
where \(\Ccal\subset \Xcal\) is a subset of much smaller cardinality and \(w:\Ccal\to\mathbb{R}\) is a suitable weight function.
Note that $L_{\mathcal{C}}(f) = \Lambda_{\Ccal}(fw)$ is a linear statistic of the configuration $\mathcal{C}$, i.e. is of the form \eqref{e:linear_statistic}.
Let $\varepsilon>0$, we say that \(\Ccal\) is an \(\varepsilon\)-coreset for \(\Xcal\) (with respect to \(\Fcal\)) if the relative error of $L_{\Ccal}(f)$ is less than $\varepsilon$, uniformly in $\mathcal{F}$.
For more on coresets, we refer to \cite{bachem-coreset}.

To derive a coreset guarantee from our results, we first assume that the data set is $\Xcal = \{X_1,\ldots,X_N\}$ consisting of i.i.d. samples from a density $\rho$
such that $\rho \in C^1([0,1]^d)$ and $\rho(\bfx)\ge \rho_{\min}>0$ on $[0,1]^d$.
For brevity, we assume the test functions $f\in\mathcal{F}$ belong to $C^{0,s}$, the case of $H^s$ being analogous.
For simplicity, we also restrict to the kernel $K$ in \eqref{eq:wavelet-dpp-kernel} with the \emph{Haar} scaling function $\phi_{\mathrm{Haar}}$, which allows us to manipulate the simpler linear statistic \eqref{e:first_MC_estimator} instead of \eqref{eq:adjusted-linear-stat}.

Following Section~\ref{sec:discreteDPP}, given the data set $\Xcal$, we use the Haar wavelets to define \eqref{e:Phi} and \eqref{e:L}, resulting in a DPP $\Scal_{\disc}$ with kernel $\mathbf{K}$ projecting onto the eigenvectors of \eqref{e:L}.
Note that, while Proposition~\ref{p:stratified_sampling} states that the Haar-wavelet DPP \emph{is} stratified sampling on dyadic hypercubes, the discrete DPP $\mathcal{S}_\mathrm{d}$ is a more sophisticated object that provides a {\it novel mechanism dedicated to the problem subset sampling from large datasets}.
Finally, consider the estimator for $L(f)$ defined by
\[L_{\Scal_{\disc}}(f) := \sum_{x \in \Scal_{\disc}}
\frac{f(x)}{\P(x \in \Scal_{\disc}|\Xcal)}, \]
where $\P(x \in \Scal_{\disc}|\Xcal)$ is the diagonal element of the matrix $\mathbf{K}$ indexed by $x$.

\begin{theorem} [Fast MSE for coresets] \label{thm:discrete-wavelet-coreset}
    Let $f\in C^{0,s}$ satisfy \eqref{eq:cpt-supp} and $\|f\|_{\infty} \le 1, |f|_{C^{0,s}} \le 1$. 
    Then, given the data set $\Xcal$, $L_{\Scal_{\disc}}(f)$ is an unbiased estimator of $L(f)$. 
    Moreover, 
    there exists a constant $C>0$ depending only on $d$ and $\rho$ such that: for all $\delta,\delta' \in (0,1)$, for all $N$ large enough such that $\widetilde{\mathscr{E}}(n,N,\delta,\delta') \le n^{-1-2s/d}$, with probability at least $1-\delta-\delta'$  in the data set $\Xcal$, we have
     \[  \E\Big [\Big |\frac{L_{\Scal_{\disc}}(f)-L(f)}{N}\Big |^2 \Big | \Xcal \Big ]   \le C n^{-1-2s/d},\]
     where
    \begin{equation} \label{eq:tilde-err}
    \widetilde{\mathscr{E}}(n,N,\delta,\delta') 
    = 4\sqrt{\frac{2\log (2/\delta')}{N}} + \frac{n+4}{9N^2} \log^2 \Big (\frac{n^2+1}{\delta}\Big )+ \frac{n+4}{N} \log \Big (\frac{n^2+1}{\delta}\Big ) \cdot  
    \end{equation} 
\end{theorem}

Theorem~\ref{thm:discrete-wavelet-coreset} shows that, under mild regularity assumptions on \(f\), 
the statistic \(L_{\Scal_{\disc}}(f)\) is an unbiased estimator of \(L(f)\) with fast MSE decay. On the other hand, \cite{coreset-dpp} develops a general machinery for discrete DPPs that converts such MSE decay estimates into coreset guarantees, i.e. to a guarantee of low mean square relative error \emph{uniformly} in $\mathcal{F}$. This is the celebrated {\it Probably Approximately Correct} (PAC)-style guarantee that is well-suited to underpin practical ML procedures.
In Appendix~\ref{app:coreset-guarantee-wavelet}, we apply this machinery to the bound in Theorem~\ref{thm:discrete-wavelet-coreset} and give an explicit coreset guarantee.

\section{Experiments} \label{sec:experiments}
We first extend an experiment on coresets for $k$-means from \cite{TrBaAm19,coreset-dpp}.
Then we demonstrate DPP minibatches in SGD applied to a non-smooth loss.
In Appendix~\ref{a:additional_experiments}, we provide quadrature experiments.\footnote{
    Our code will be made publicly available upon publication.
}

\subsection{A controlled experiment on $k$-means}

We compare five methods. 
Given an input dataset $\mathcal{X}$ and an integer $m$, each method returns a random coreset of size $m$. 
The baseline is i.i.d. uniform sampling, and we implement four DPP samplers. 
The first DPP is the OPE built on the empirical measure of the dataset $\mathcal{X}$, denoted \textsc{vdm-dpp} as in the original reference \cite{TrBaAm19}. 
The other three DPP samplers are built using our continuous-to-discrete pipeline of Section~\ref{sec:discreteDPP}, namely Equation \eqref{e:LHat}.
We take $\hat{\rho}$ to be the kernel density estimator (KDE) built on $\mathcal{X}$ using the Epanechnikov kernel, with Scott's bandwidth selection method.
The three methods differ by the underlying continuous DPP Kernel they use: \textsc{ope} uses the multivariate OPE kernel of \cite{bardenet-aap}, \textsc{haar} uses the wavelet DPP kernel \eqref{eq:wavelet-dpp-kernel} with the Haar scaling function, and \textsc{db2} the wavelet DPP kernel with the popular\footnote{
 For wavelets, we use the Python package \href{https://github.com/PyWavelets/pywt}{PyWavelets} \cite{Lee2019}, MIT Licence.
}
 \emph{periodized Daubechies-2} scaling function \cite{db92, Mal08}. 
Note that for this \textsc{db2} sampler, we use the adjusted linear statistic $\tilde{\mathcal{L}}_{\mathcal{S}_d}$ in \eqref{eq:adjusted-linear-stat} to estimate the full empirical loss. 

Our performance metric is $Q_{\mathcal{S}}(0.9)$, where $Q_{\mathcal{S}}$ is the quantile function of the relative error $\sup |L_{\mathcal{S}_d}(f) - L(f)|/|L(f)$. While evaluating the supremum of the relative error, we uniformly sample without replacement $k$ elements of $\mathcal{X}$ 150 times, and we take the maximum value of the relative error among these values. Moreover, for each method and each coreset size $m$, the quantile is estimated by an empirical quantile over 150 independent coresets.

\paragraph{Results.} We first conduct our experiment using a trimodal synthetic dataset of 1024 points supported on $[-1,1]^2$. 
The dataset is drawn from a balanced Gaussian mixture on $\mathbb{R}^2$, with centers at $(\pm \sqrt{2}/2, 
\pm \sqrt{2}/2)$ and the origin. 
The second dataset is the MNIST dataset,\footnote{made available under the terms of the Creative Commons Attribution-Share Alike 3.0 license.} and reduced to $d =2$ using PCA. The experiment is performed with $k = 3$ for the trimodal dataset, and with $k = 10$ for the MNIST dataset.
Results are shown in Figure~\ref{fig:clustering-results}.
For both datasets, the DPP samplers beat uniform sampling, as expected. 
In the case of MNIST, both \textsc{haar} and \textsc{db2} improve over the OPE samplers, with \textsc{vdm-dpp} slightly outperforming \textsc{ope}. 
In the case of the trimodal dataset, the performance of \textsc{haar} regresses to that of the OPE samplers, while the improvement offered by \textsc{db2} persists. 
Overall, the experiment supports the recommendation to use wavelet-based samplers.

\begin{figure}[!htbp]
    \centering

    \begin{minipage}{0.45\textwidth}
        \centering
        \includegraphics[width=\textwidth]{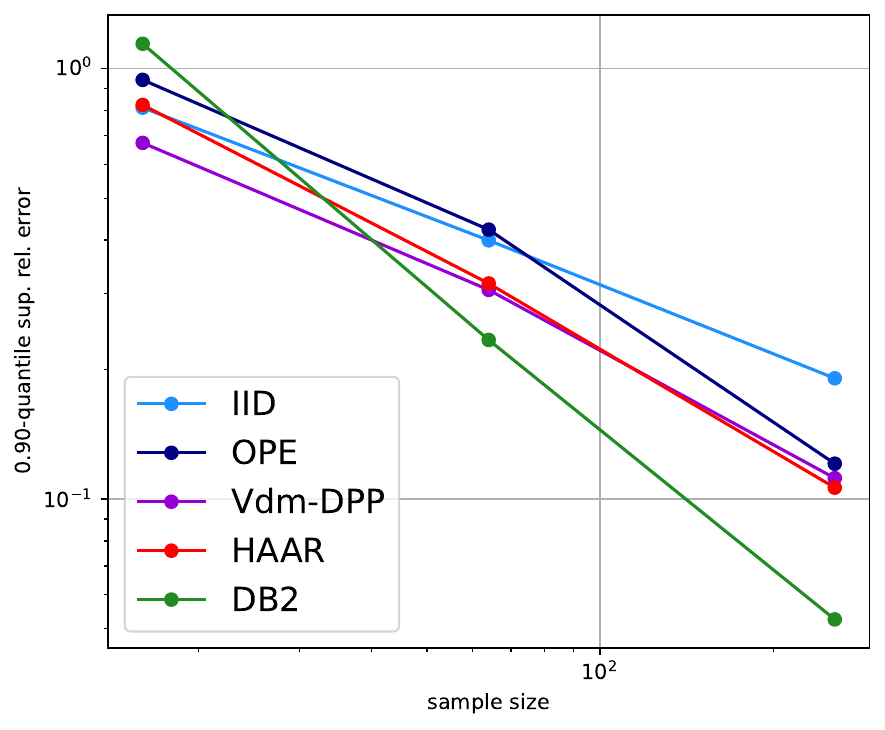}
        \caption{$Q_{\mathcal{S}}(0.9)$ vs. $m$: Trimodal Dataset}
    \end{minipage}
    \hfill
    \begin{minipage}{0.45\textwidth}
        \centering
        \includegraphics[width=\textwidth]{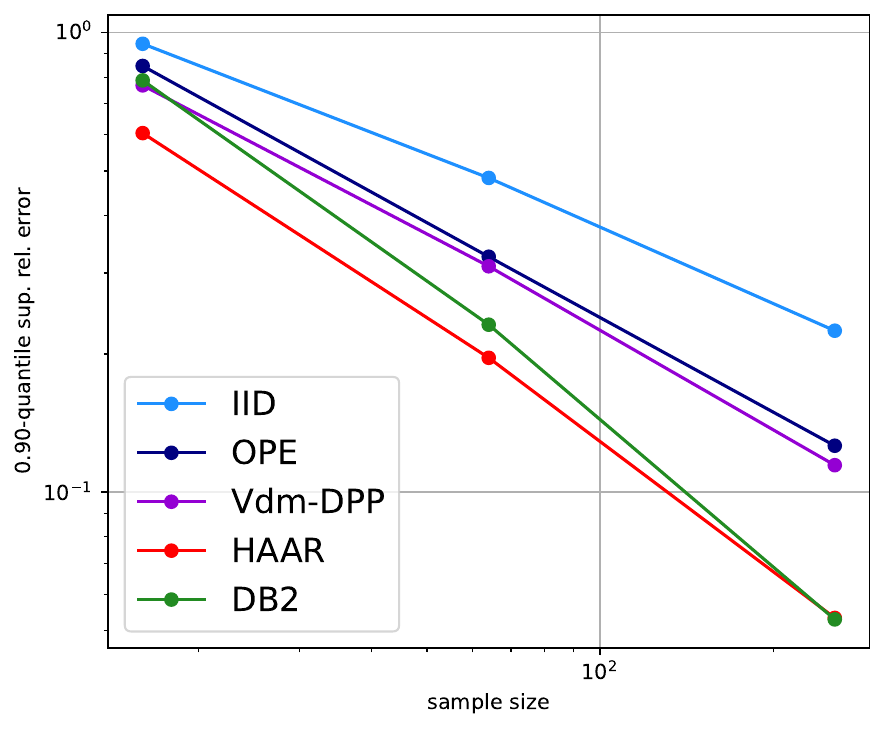}
        \caption{$Q_{\mathcal{S}}(0.9)$ vs. $m$: MNIST Dataset}
    \end{minipage}
    \label{fig:clustering-results}
\end{figure}

\subsection{Classification of MNIST with a non-smooth loss} \label{sec:MNIST}
To demonstrate the benefits of coupling wavelet-based DPPs and our discretization mechanism in minibatch sampling for non-smooth losses, we consider the \emph{Pegasos} algorithm \cite{shalev2007pegasos}.
\emph{Pegasos} is sub-gradient descent on the hinge loss, i.e. for basic SVMs. 
The sub-gradient at $\theta_t$
is given by
\begin{equation*}
    \nabla_t = \lambda \theta_t + \frac{1}{|\mathcal{S}_d|}\sum_{i \in \mathcal{S}_d} \mathbf{1}\{y_i\langle\theta_t, x_i\rangle < 1\}y_ix_i.
\end{equation*}
We compute this using minibatches $\mathcal{S}_d$
produced by four samplers: uniform sampling (\textsc{iid}), which we use as a baseline, and three DPP-based samplers built using \eqref{e:LHat} again, with $\hat\rho$ given by a Gaussian KDE on $\mathcal{X}$ with Scott's bandwidth.
The DPP samplers only vary by the kernel of the underlying continuous kernel: the kernel of the multivariate OPE is labeled as \textsc{ope} in Figure~\ref{fig:sgd-three-plots}, the wavelet-based kernel with the Haar scaling function yields \textsc{haar}, and \textsc{db2} corresponds to the periodized Daubechies-2 scaling function. 
With the \textsc{db2} sampler, we again use the adjusted linear statistic \eqref{eq:adjusted-linear-stat}.

\paragraph{Results.} We use the (balanced) dataset $\mathcal{X}$ of $N = 1000$, of 4s and 9s in MNIST, which we reduce to $d = 2$ using PCA. We use a train-test split of 70:30.
For each sampler, we form two minibatches of size $m = 16$, one from each class. This gives a total minibatch size of $m = 32$. 
Pegasos is run for $T = 200$ iterations with regularization $\lambda = 0.1$ and step size $\eta_t = 1/(\lambda t)$.
Our performance metrics are the test classification error, the full-batch sub-gradient norm $\|\nabla f(\theta_t)\|_2$, and the parameter-error $\|\theta_t - \theta^\star\|_2$. 
The reference solution $\theta^\star$ is obtained using Scikit-learn's linear SVM solver \texttt{sklearn.svm.SVC} \cite{scikit-learn}.
The performance metrics are averaged over 100 independent trials for each sampler, and the error bars are for $\pm 2$ standard errors.

The DPP samplers offer considerable improvement over the iid baseline on all metrics. 
Moreover, both wavelet DPPs also outperform \textsc{ope}, with the advantage being the most significant for \textsc{db2}.

\begin{figure*}[!htbp]
    \centering

    \begin{minipage}{0.32\textwidth}
        \centering
        \includegraphics[width=\linewidth]{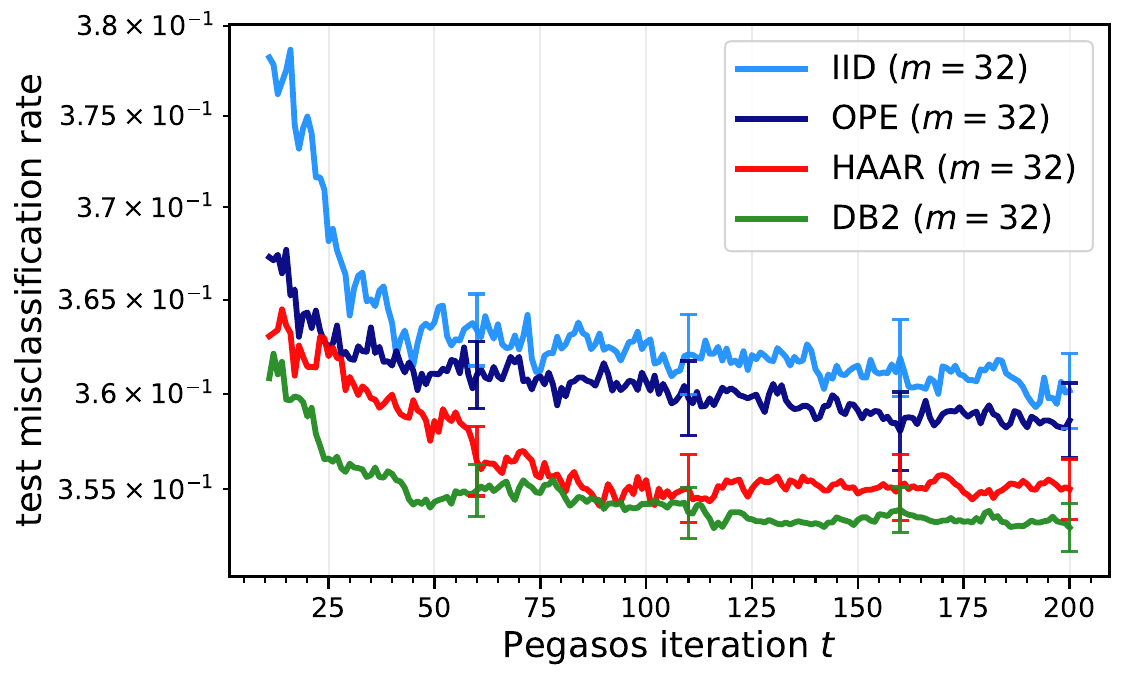}
        \caption{Test Error}
    \end{minipage}
    \hfill
    \begin{minipage}{0.32\textwidth}
        \centering
        \includegraphics[width=\linewidth]{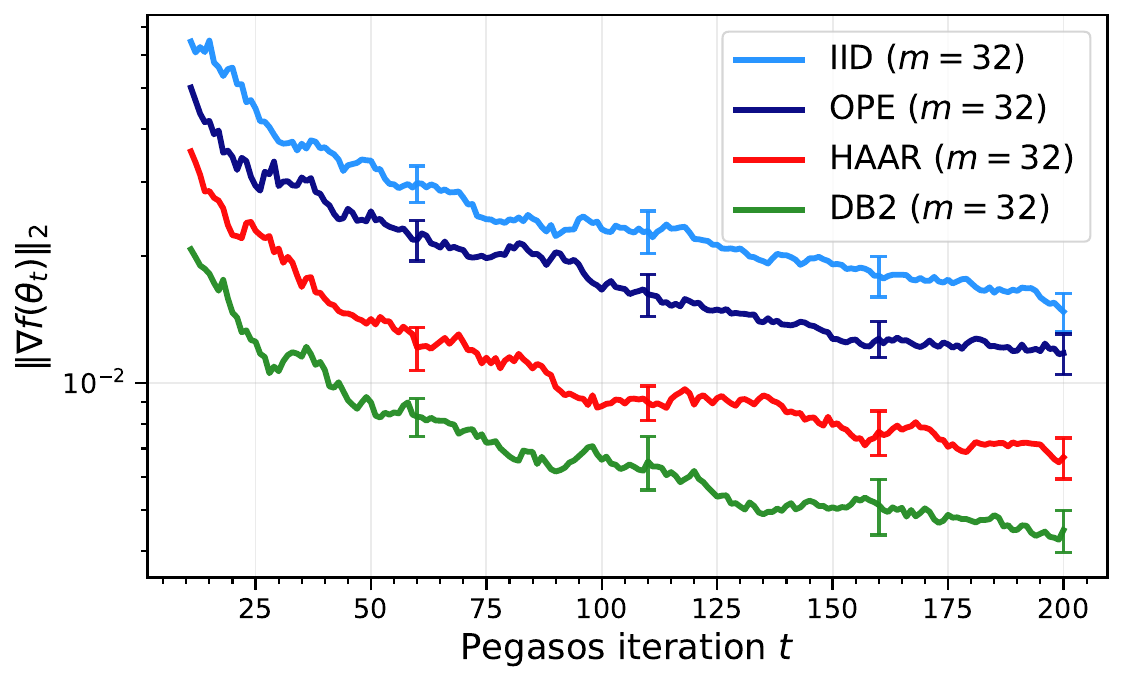}
        \caption{$\|\nabla f(\theta_t)\|_2$}
    \end{minipage}
    \hfill
    \begin{minipage}{0.32\textwidth}
        \centering
        \includegraphics[width=\linewidth]{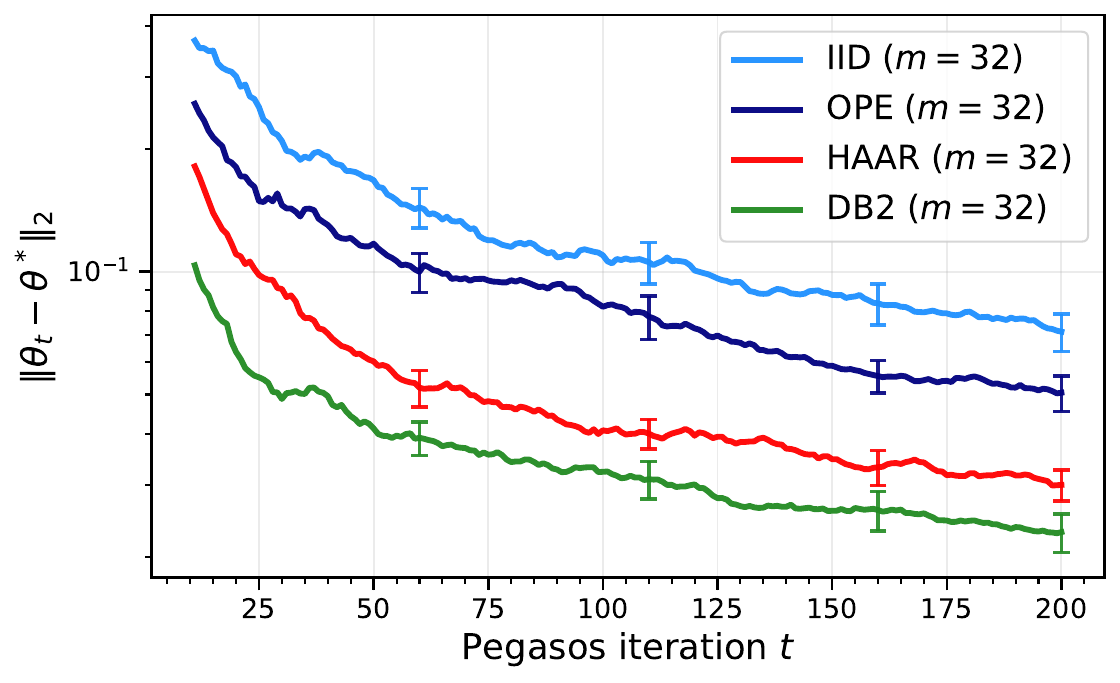}
        \caption{$\|\theta_t - \theta^*\|_2$}
    \end{minipage}
    \label{fig:sgd-three-plots}
\end{figure*}

\section{Discussion} \label{sec:discussion}
This work contributes to the programme of developing DPPs as a subsampling toolbox for ML in multiple ways. 
First, we propose new DPPs on the Euclidean space based on wavelets, with provably better accuracy guarantees than the best known rates. In this vein, we are able to address classes of highly irregular functions, with a view to ML applications with non-smooth objective functions (e.g. coresets with non-differentiable loss functions)  and rough data classes (such as financial data). Second, we introduce a general pipeline to convert such continuous DPPs, which are more amenable to proving analytical estimates, into discrete kernels, which are pertinent for subsampling tasks such as minibatch and coreset constructions.  
This very general conversion pipeline, acting on an input continuous kernel in a completely turn-key fashion, simultaneously preserves the desired variance decay and automatically reveals a low-rank decomposition of the discrete kernel, which makes sampling the corresponding DPP computationally inexpensive. 
This work also raises natural questions for further investigation -- principal among them being how much sharper one can push the variance reduction properties of determinantal samplers through more efficient kernel designs. 
An associated question is that of improving the dependence of the rates on the effective dimension.
Our results on non-smooth objective functions raise natural questions on
optimal dependencies of variance reduction rates on their regularity. 
Finally, and more specific to the wavelet-based DPP, it would also be practically  relevant to decouple the cardinality of the DPP from the dimension.

\section*{Acknowledgement \& Funding}
HST was supported by the Singapore MOE  grant A-8003576-00-00.
SG was supported in part by the NUS Dean's Chair Associate Professorship E-146-00-0037-01 and the Singapore MOE grants A-8002014-00-00 and A-8003802-00-00.
\bibliographystyle{plainnat}
\bibliography{contents/reference,contents/stats}

@article{tropp2015,
  title={An introduction to matrix concentration inequalities},
  author={Tropp, J.},
  journal={Foundations and trends{\textregistered} in machine learning},
  volume={8},
  number={1-2},
  pages={1--230},
  year={2015},
  publisher={Emerald Publishing Limited}
}

@article{mcdiarmid1989,
  title={On the method of bounded differences},
  author={McDiarmid, C.},
  journal={Surveys in combinatorics},
  volume={141},
  number={1},
  pages={148--188},
  year={1989},
  publisher={Norwich}
}

@article{acosta2017,
  title={A fractional {L}aplace equation: regularity of solutions and finite element approximations},
  author={Acosta, G. and Borthagaray, J.P.},
  journal={SIAM Journal on Numerical Analysis},
  volume={55},
  number={2},
  pages={472--495},
  year={2017},
  publisher={SIAM}
}

@inproceedings{shalev2007pegasos,
  title={Pegasos: Primal estimated sub-gradient solver for svm},
  author={Shalev-Shwartz, S. and Singer, Y. and Srebro, N.},
  booktitle={Proceedings of the 24th international conference on Machine learning},
  pages={807--814},
  year={2007}
}

@InProceedings{gillenwater19a,
  title = 	 {A Tree-Based Method for Fast Repeated Sampling of Determinantal Point Processes},
  author =       {Gillenwater, J. and Kulesza, A. and Mariet, Z. and Vassilvtiskii, S.},
  booktitle = 	 {Proceedings of the 36th International Conference on Machine Learning},
  pages = 	 {2260--2268},
  year = 	 {2019},
  editor = 	 {Chaudhuri, Kamalika and Salakhutdinov, Ruslan},
  volume = 	 {97},
  series = 	 {Proceedings of Machine Learning Research},
  month = 	 {09--15 Jun},
  publisher =    {PMLR},
  url = 	 {https://proceedings.mlr.press/v97/gillenwater19a.html}
}

@article{bardenet-aap,
author = {R. Bardenet and A. Hardy},
title = {{Monte Carlo with determinantal point processes}},
volume = {30},
journal = {The Annals of Applied Probability},
number = {1},
publisher = {Institute of Mathematical Statistics},
pages = {368 -- 417},
year = {2020},
doi = {10.1214/19-AAP1504},
URL = {https://doi.org/10.1214/19-AAP1504}
}

@article{dpp-sgd,
  title={Determinantal point processes based on orthogonal polynomials for sampling minibatches in SGD},
  author={Bardenet, R. and Ghosh, S. and Lin, M.},
  journal={Advances in Neural Information Processing Systems},
  volume={34},
  pages={16226--16237},
  year={2021}
}

@misc{jaquard2026,
      title={Statistical Consistency of Discrete-to-Continuous Limits of Determinantal Point Processes}, 
      author={H. Jaquard and N. Keriven},
      year={2026},
      eprint={2603.01670},
      archivePrefix={arXiv},
      primaryClass={math.PR},
      url={https://arxiv.org/abs/2603.01670}, 
}

@article{coreset-dpp,
  title={Small coresets via negative dependence: DPPs, linear statistics, and concentration},
  author={Bardenet, R. and Ghosh, S. and Simon-Onfroy, H. and Tran, H.S.},
  journal={Advances in Neural Information Processing Systems},
  volume={37},
  pages={84329--84349},
  year={2024}
}

@article{bachem-coreset,
  title={Practical coreset constructions for machine learning},
  author={Bachem, O. and Lucic, M. and Krause, A.},
  journal={arXiv preprint arXiv:1703.06476},
  year={2017}
}

@InProceedings{kde-jiang-icml,
  title = 	 {Uniform Convergence Rates for Kernel Density Estimation},
  author =       {H. Jiang},
  booktitle = 	 {Proceedings of the 34th International Conference on Machine Learning},
  pages = 	 {1694--1703},
  year = 	 {2017},
  editor = 	 {Precup, Doina and Teh, Yee Whye},
  volume = 	 {70},
  series = 	 {Proceedings of Machine Learning Research},
  month = 	 {06--11 Aug},
  publisher =    {PMLR},
  url = 	 {https://proceedings.mlr.press/v70/jiang17b.html}
}

@article{payne60,
	author = {Payne, L. E. and Weinberger, H. F.},
	date = {1960/01/01},
	doi = {10.1007/BF00252910},
	id = {Payne1960},
	isbn = {1432-0673},
	journal = {Archive for Rational Mechanics and Analysis},
	number = {1},
	pages = {286--292},
	title = {An optimal Poincar{\'e}inequality for convex domains},
	url = {https://doi.org/10.1007/BF00252910},
	volume = {5},
	year = {1960}}

@article{Lee2019, 
    doi = {10.21105/joss.01237}, 
    url = {https://doi.org/10.21105/joss.01237}, 
    year = {2019}, 
    publisher = {The Open Journal}, 
    volume = {4}, 
    number = {36}, 
    pages = {1237}, 
    author = {Lee, G. R. and Gommers, R. and Waselewski, F. and Wohlfahrt, K. and O'Leary, A.}, 
    title = {PyWavelets: A Python package for wavelet analysis}, 
    journal = {Journal of Open Source Software} 
    }

@article{LeBa24Sub,
	author = {Lemoine, T. and Bardenet, R.},
	journal = {arXiv preprint arXiv:2405.09203},
	title = {Monte {C}arlo methods on compact complex manifolds using {B}ergman kernels},
	year = {2024}}

@phdthesis{Mac72,
	author = {Macchi, O.},
	school = {Universit\'e Paris-Sud},
	title = {Processus ponctuels et coincidences -- Contributions {\`a} l'{\'e}tude th{\'e}orique des processus ponctuels, avec applications {\`a} l'optique statistique et aux communications optiques},
	year = {1972}}

@article{TrBaAm19,
	author = {Tremblay, N. and Barthelm\'e, S. and Amblard, P.-O.},
	journal = {Journal of Machine Learning Research},
	number = {168},
	pages = {1--70},
	title = {Determinantal Point Processes for Coresets.},
	volume = {20},
	year = {2019}}

@article{CoMaAm20Sub,
	author = {Coeurjolly, J.-F. and Mazoyer, A. and Amblard, P.-O.},
	journal = {arXiv preprint arXiv:2003.10323},
	title = {{M}onte {C}arlo integration of non-differentiable functions on $[0,1]^\iota$, $\iota=1,\dots,d$, using a single determinantal point pattern defined on $[0,1]^d$},
	year = {2020}}

@inproceedings{BeBaCh19,
	author = {Belhadji, A. and Bardenet, R. and Chainais, P.},
	booktitle = {Advances in Neural Information Processing Systems (NeurIPS)},
	journal = {\href{https://arxiv.org/abs/1906.07832}{Arxiv preprint:1906.07832}},
	title = {Kernel quadrature with determinantal point processes},
	year = {2019}}

@book{RoCa04,
	author = {Robert, C. P. and Casella, G.},
	publisher = {Springer},
	title = {Monte {C}arlo statistical methods},
	year = {2004}}

@article{haber1966modified,
  title={A modified Monte-Carlo quadrature},
  author={Haber, S.},
  journal={Mathematics of Computation},
  volume={20},
  number={95},
  pages={361--368},
  year={1966},
  publisher={JSTOR}
}

@inproceedings{barthelme2023faster,
  title={A faster sampler for discrete determinantal point processes},
  author={Barthelm\'e, Simon and Tremblay, Nicolas and Amblard, Pierre-Olivier},
  booktitle={International Conference on Artificial Intelligence and Statistics},
  pages={5582--5592},
  year={2023},
  organization={PMLR}
}

@article{Mac75,
	author = {Macchi, O.},
	journal = {Advances in Applied Probability},
	pages = {83--122},
	title = {The coincidence approach to stochastic point processes},
	volume = {7},
	year = {1975}}

@article{TBUA23,
  title={Extended L-ensembles: a new representation for Determinantal Point Processes},
  author={Tremblay, N. and Barthelm{\'e}, S. and Usevich, K. and Amblard, P.-O.},
  journal={The Annals of Applied Probability},
  volume={33},
  number={1},
  pages={613--640},
  year={2023},
  publisher={Institute of Mathematical Statistics}
}

@book{Mal08,
  title={A wavelet tour of signal processing},
  edition={Third},
  year={2008},
  author={Mallat, S.},
  publisher={Elsevier}
}

@article{HKPV06,
	author = {Hough, J. B. and Krishnapur, M. and Peres, Y. and Vir\'ag, B.},
	journal = {Probability surveys},
	title = {Determinantal processes and independence},
	year = {2006}}

@article{tran2025negative,
  title={Negative Dependence as a toolbox for machine learning: review and new developments},
  author={Tran, H.S. and Petrovic, V. and Bardenet, R. and Ghosh, S.},
  journal={arXiv preprint arXiv:2502.07285},
  year={2025}
}

@article{KuTa12,
	author = {Kulesza, A. and Taskar, B.},
	journal = {Foundations and Trends in Machine Learning},
	title = {Determinantal point processes for machine learning},
	year = {2012}}

@article{scikit-learn,
  author  = {Pedregosa, F. and Varoquaux, G. and Gramfort, A. and Michel, V.
             and Thirion, B. and Grisel, O. and Blondel, M. and Prettenhofer, P.
             and Weiss, R. and Dubourg, V. and Vanderplas, J. and Passos, A.
             and Cournapeau, D. and Brucher, M. and Perrot, M. and Duchesnay, E.},
  title   = {Scikit-learn: Machine Learning in {P}ython},
  journal = {Journal of Machine Learning Research},
  volume  = {12},
  pages   = {2825--2830},
  year    = {2011},
}

@book{db92,
author = {Daubechies, I.},
title = {Ten Lectures on Wavelets},
publisher = {Society for Industrial and Applied Mathematics},
year = {1992},
doi = {10.1137/1.9781611970104},
address = {},
edition   = {},
URL = {https://epubs.siam.org/doi/abs/10.1137/1.9781611970104},
eprint = {https://epubs.siam.org/doi/pdf/10.1137/1.9781611970104}
}

\newpage

\appendix

\section{Technical appendices to discrete-to-continuous DPPs} \label{a:appendix}
In this section, we provide additional comments as well as the detailed proofs of our theorems in Section \ref{sec:discreteDPP}. 

\subsection{Additional comments on Theorem~\ref{thm:disc-cont-true}}
\label{s:more_remarks_on_the_theorem}
\paragraph{On the order of $K_{\max}$ in \eqref{eq:Kmax}.}
    Since $\int K(x,x)\mu(\d x)=n$, we have $K_{\max}\ge n$. In many examples of interest, one in fact has $K_{\max}\asymp n$. For instance, in orthogonal polynomial ensembles, if $\supp(\rho)$ is contained in a compact subset of the bulk and the reference measure $\mu$ is regular enough so that uniform bulk Christoffel asymptotics hold there, then $\sup_{x\in \supp(\rho)} K(x,x)\asymp n$ (see \cite{bardenet-aap}, Theorem A.4); in addition, if $\rho$ is bounded below on its support, we will have $K_{\max}\asymp n$.
    Another family of examples is our novel class of wavelet DPPs introduced in Section \ref{sec:waveletDPP}. We will see that $K_{\max}\asymp n$ in these examples as well.

\paragraph{On the minibatch size.}
   As remarked earlier, our algorithm produces a minibatch \(\Scal_{\disc}\) of size \(m \le n\), rather than one that is necessarily of size exactly \(n\). However, we note that if \(N\) is sufficiently large so that $K^2_{\max} \mathscr{E}(n,N,\delta,\delta') <1$ (which is the relevant regime in practice), then on the same event appearing in the conclusion of Theorem~\ref{thm:disc-cont-true}, the minibatch has size exactly \(n\). This follows from our proof technique; see Lemma~\ref{lm:bernstein}. 

\paragraph{On explicit error terms.}
A notable feature of our result is that the error term \(\mathscr{E}(n,N,\delta,\delta')\) is completely explicit and contains no hidden constants, which is particularly useful for quantitative analysis. It also distinguishes our result from many existing works in the literature, which often provide only asymptotic guarantees or estimates with implicit constants. In particular, existing discrete-to-continuous results typically treat the regime where \(N\to\infty\) while the minibatch size \(n\) remains fixed. By contrast,
our framework allows both the data set size \(N\) and the minibatch size \(n\) to grow simultaneously, with an explicit control of the corresponding error terms.

\paragraph{On the regime $N\gg n$.}    
Theorem \ref{thm:disc-cont-true} establishes that our minibatch sampling algorithm satisfies a \emph{variance-transfer guarantee}: the conditional variance of the linear statistic associated with the minibatch is bounded above and below by that of the underlying continuous DPP, up to \emph{explicit} error terms. Therefore, our algorithm does not merely produce a repulsive minibatch; it approximately preserves the key second-order quantity (variance of linear statistics) of the underlying continuous DPPs, which is useful in sampling.
    
    In the regime most relevant for applications, where the full data size $N$ is much larger than the target minibatch size $n$, these error terms are negligible. Indeed, we observe that \(K_{\max}\asymp n\) and $\mathscr{E}(n,N,\delta,\delta') = O(N^{-1/2})$ for large $N$. Thus, the error term is of order \(O(n^2/N)\) in the expectation bound, and it is of order \(O(n^2/\sqrt{N})\) in the high probability bound. Therefore, in the large $N$ regime, the dominant term is \(\var[\L_{\Scal_{\cont}}(f)]\).
    By choosing a continuous DPP kernel $K$ that is known to enjoy variance reduction (i.e., \(\var[\L_{\Scal_{\cont}}(f)]=O(n^{1-\alpha})\) for some \(\alpha>0\)), the resulting minibatch $\Scal_{\disc}$ automatically inherits this variance reduction property. 

\subsection{Proofs of results in Section \ref{sec:discreteDPP}}

For simplicity of notation, we write the proofs in the real-valued case, namely, when $\psi_k$'s are real-valued. 
The complex-valued case follows by the same argument after systematically replacing transposes by Hermitian adjoints and inserting complex conjugates where appropriate.

We set up some notations. Let $f:E\rightarrow \R$ be a bounded test function, we define
\begin{equation} \label{eq:def-F}
    F(X_1,\ldots,X_N):= \frac{1}{N^2}\sum_{1\le i,j \le N} |f(X_i)-f(X_j)|^2 \frac{|K(X_i,X_j)|^2}{\rho(X_i)\rho(X_j)} \cdot 
\end{equation}

For a function $\varphi$ on $\Xcal$, we set 
\[ \bfD(\varphi):= \diag (\varphi(X_1),\ldots,\varphi(X_N))\]
where $\diag(a_1,\ldots,a_N)$ denotes the $N\times N$-diagonal matrix with the diagonal entries are given by $(a_1,\ldots,a_N)$. For a square matrix $\bfA$, we denote by $\|\bfA\|_{\Frob}$ the Frobenius norm of $\bfA$.

For $n,N \in \N$ and $\delta,\delta' \in (0,1)$, we define
\begin{equation} 
     \mathscr{E}(n,N,\delta,\delta') = 4\sqrt{\frac{2\log (2/\delta')}{N}}+ \frac{4}{9N^2} \log^2 \Big (\frac{n^2+1}{\delta}\Big )+ \frac{4}{N} \log \Big (\frac{n^2+1}{\delta}\Big ).
\end{equation}





\subsubsection{Some technical lemmas}

\begin{lemma}[Structure of $r$-DPP] \label{lm:n-dpp}
    Let $[N]=\{1,2,\ldots,N\}$ and $r\le N$. Let $\Ycal$ be an $L$-ensemble on $[N]$ with an $L$-kernel $\mathbf{L}$ of rank $r$
    \[ \mathbf{L} = \sum_{i=1}^r \lambda_i \mathbf{v}_i \mathbf{v}_i^{\top}, \quad \lambda_i>0,\]
    where $\{\mathbf{v}_i:1\le i \le r\}$ is orthonormal in $\R^N$. Then:
    \begin{itemize}
        \item[(i)]  The $r$-DPP defined by $\mathbf{L}$, denoted by $\Ycal_{\ndpp}$, is also a projection DPP on $[N]$ with kernel given by
    \[\mathbf{K} = \sum_{i=1}^r \mathbf{v}_i \mathbf{v}_i^{\top}. \]
        \item[(ii)]  Moreover, we have
    \[ | \bfL_{jj} - \bfK_{jj}| \le \Big( \max_{1\le i \le r} |\lambda_i-1| \Big) \bfK_{jj}, \quad \forall 1\le j \le N.\]
    \end{itemize} 
\end{lemma}

\begin{proof}
(i)
    Let $J\subset [N]$ with $|J|=r$,  we have the following formula:
    \[ \P(\Ycal_{\ndpp} = J) = \frac{1}{\lambda_1\ldots\lambda_r} \det[\mathbf{L}+\mathbf{I}] \P(\Ycal = J) 
    = \prod_{i=1}^r \frac{\lambda_i +1}{\lambda_i} \P(\Ycal =J). \]
    Since $\Ycal$ is an $L$-ensemble on $[N]$ with the kernel $\bfL$, $\Ycal$ is also a DPP on $[N]$ with the kernel given by
    \[ \bfL(\mathbf{I}+\bfL)^{-1} = \sum_{i=1}^r \frac{\lambda_i}{\lambda_i+1}\mathbf{v}_i\mathbf{v}_i^{\top}. \]
    Thus, $|\Ycal| \le r$ almost surely. In particular, for any $J\subset [N]$ with $|J| = r$, we must have 
    \[\P(\Ycal = J) = \P(\Ycal \supset J).\] 
    Thus
    \[ \P(\Ycal_{\ndpp} = J) 
    = \Big (\prod_{i=1}^r \frac{\lambda_i+1}{\lambda_i} \Big ) 
    \cdot \det \Big [ \Big(\sum_{i=1}^r \frac{\lambda_i}{\lambda_i+1}\mathbf{v}_i\mathbf{v}_i^{\top}\Big)_{J\times J}\Big ].\]
    We now claim that for every $J \subset [N]$ with $|J|=r$
    \begin{equation} \label{eq:det-submatrix}
         \Big (\prod_{i=1}^r \frac{\lambda_i+1}{\lambda_i} \Big ) \cdot 
         \det \Big [ \Big(\sum_{i=1}^r \frac{\lambda_i}{\lambda_i+1}\mathbf{v}_i\mathbf{v}_i^{\top}\Big)_{J\times J}\Big ] =
    \det \Big [ \Big(\sum_{i=1}^r \mathbf{v}_i\mathbf{v}_i^{\top}\Big)_{J\times J}\Big ].
    \end{equation}
     Indeed, let 
    \[ \bfD := \diag \Big ( \frac{\lambda_1}{\lambda_1+1}, \ldots, \frac{\lambda_r}{\lambda_r+1}\Big ) \in \R^{r\times r} \quad,\quad \bfV = [\bfv_1,\ldots,\bfv_r] \in \R^{N\times r}.   \]
    Then
    \[ \sum_{i=1}^r \frac{\lambda_i}{\lambda_i+1} \bfv_i\bfv_i^{\top} = \bfV \bfD \bfV^{\top} = (\bfD^{1/2}\bfV^{\top})^{\top}(\bfD^{1/2}\bfV^{\top}) \in \R^{N\times N}.\]
    For $J\subset [N]$ with $|J|=r$, 
    let $(\bfD^{1/2}\bfV^{\top})_J$ denote the $r\times r$-submatrix of $\bfD^{1/2}\bfV^{\top}$ whose columns are indexed by $J$ (note that $\bfD^{1/2}\bfV^{\top} \in \R^{r\times N}$). Then it is clear that
    \[ \Big (\sum_{i=1}^r \frac{\lambda_i}{\lambda_i+1} \bfv_i\bfv_i^{\top} \Big )_{J \times J} = ((\bfD^{1/2}\bfV^{\top})_J)^{\top} (\bfD^{1/2}\bfV^{\top})_J. \]
    Taking determinants both sides gives 
    \begin{eqnarray*}
        \det \Big (\sum_{i=1}^r \frac{\lambda_i}{\lambda_i+1} \bfv_i\bfv_i^{\top} \Big )_{J \times J} &=& \Big ( \det (\bfD^{1/2}\bfV^{\top})_J \Big )^2 = \det [\bfD] \cdot\det [\bfV \bfV^{\top}]_{J\times J} \\
        &=& \Big (\prod_{i=1}^r \frac{\lambda_i}{\lambda_i+1} \Big ) \det \Big (\sum_{i=1}^r \bfv_i \bfv_i^{\top} \Big )_{J\times J}.
    \end{eqnarray*}
    as desired. The claim follows.

    Therefore, $\P(\Ycal_{\ndpp} = J) = \det [\bfK_{J\times J}]$ 
    where $\bfK:=  \sum_{i=1}^r \mathbf{v}_i \mathbf{v}_i^{\top}$ and $\bfK_{J \times J}$ is the submatrix of $\bfK$ whose rows and columns indexed by $J$. Since $\bfK$ is a projection matrix, $\Ycal_{\ndpp}$ is a projection DPP with kernel $\bfK$ as desired.

\medskip
\noindent

    (ii) 
    For each $1\le i \le r$, we write $\bfv_i = (\bfv_{i1},\ldots,\bfv_{iN})^\top$.
    Then we have
    \begin{eqnarray*}
        \bfL_{jj} = \sum_{i=1}^r \lambda_i \bfv_{ij}^2 \quad,\quad \bfK_{jj} = \sum_{i=1}^r \bfv_{ij}^2 \qquad \forall 1\le j \le N.
    \end{eqnarray*}
    This implies 
    \[ |\bfL_{jj} - \bfK_{jj}| \le \sum_{i=1}^r |\lambda_i-1| \bfv_{ij}^2 \le  \Big( \max_{1\le i \le r} |\lambda_i-1| \Big) \sum_{i=1}^r \bfv_{ij}^2 =  \Big( \max_{1\le i \le r} |\lambda_i-1| \Big) \bfK_{jj} \]
    as desired.
\end{proof}

\begin{lemma} \label{lm:mcdiarmid}
    Let $F$ be defined as in Equation \eqref{eq:def-F}. We have
    \[ \E[F(X_1,\ldots,X_N)]= \frac{2(N-1)}{N}\var[\L_{\Scal_{\cont}}(f)].\]
    Moreover, for every $t\ge 0$
    \[ \P(|F(X_1,\ldots,X_N) - \mathbb{E}[F(X_1,\ldots,X_N)]| \ge t ) \le 2\exp \Big (-\frac{Nt^2}{32\|f\|^4_{\infty}K_{\max}^4} \Big )\cdot\]
    In particular,
    for every $\delta'\in (0,1)$, with probability at least $1-\delta'$ we have
    \[ |F(X_1,\ldots,X_N)-\E[F(X_1,\ldots,X_N)]|\le  4\|f\|_{\infty}^2K_{\max}^2\sqrt{\frac{2\log (2/\delta')}{N}} \cdot\]
\end{lemma}

\begin{proof}
   We have
\begin{eqnarray*}
    \E[F(X_1,\ldots,X_N)] &=& \frac{N-1}{N} \E\Big [|f(X_1)-f(X_2)|^2\frac{|K(X_1,X_2)|^2}{\rho(X_1)\rho(X_2)}\Big ] \\
    &=& \frac{N-1}{N}\int |f(x)-f(y)|^2|K(x,y)|^2\mu(\d x)\mu(\d y) \\
    &=& \Big(2-\frac{2}{N}\Big )\var[\L_{\Scal_{\cont}}(f)].
\end{eqnarray*}
    For every $1\le i \le N$, we have
    \begin{eqnarray*}
        &&|F(x_1,\ldots,x_i,\ldots,x_N) - F(x_1,\ldots,x_i',\ldots,x_N)| \\&\le& \frac{2}{N^2}\sum_{j=1}^N4\|f\|^2_{\infty}\max\Big (\frac{|K(x_i,x_j)|^2}{\rho(x_i)\rho(x_j)},\frac{|K(x_i',x_j)|^2}{\rho(x_i')\rho(x_j)}\Big )\le\frac{8\|f\|^2_{\infty}K_{\max}^2}{N} ,
    \end{eqnarray*}
    where we used that $|K(x,y)|^2 \le K(x,x) K(y,y)$. 
    By the McDiarmid's inequality for functionals of independent random variables satisfying the bounded differences property (see \cite{mcdiarmid1989}), we have for every $t\ge 0$
    \[ \P(|F(X_1,\ldots,X_N) - \mathbb{E}[F(X_1,\ldots,X_N)]| \ge t ) \le 2\exp \Big (-\frac{Nt^2}{32\|f\|^4_{\infty}K_{\max}^4} \Big )\cdot\]
    Let $\delta' = 2\exp \Big (-\frac{Nt^2}{32\|f\|^4_{\infty}K_{\max}^4} \Big ) $, we get 
    \[ t = \|f\|^2_{\infty} K_{\max}^2\sqrt{\frac{32\log(2/\delta')}{N}} \]
    as desired.
\end{proof}

\begin{lemma}[Structures of $\bfL$ and $\widehat{\bfL}$] \label{lm:L-structure}
    For $1\le i,j \le N$, we have
    \begin{equation*}
        \bfL_{ij} = \frac{K(X_i,X_j)}{N\sqrt{\rho(X_i) \rho(X_j)}}  \quad ,  \quad \widehat{\bfL}_{ij} = \frac{K(X_i,X_j)}{N\sqrt{\hat\rho(X_i) \hat\rho(X_j)}} \cdot
    \end{equation*}
\end{lemma}

\begin{proof}
    Since $K(x,y) = \sum_{k=1}^n \psi_k(x)\psi_k(y)$, we have
    \[ 
    \begin{pmatrix}
        K(X_1,X_1) & \ldots & K(X_1,X_N) \\
        \vdots & \ddots & \vdots \\
        K(X_N,X_1) & \ldots & K(X_N,X_N)
    \end{pmatrix}
    = 
    \begin{pmatrix}
        \psi_1(X_1) & \ldots & \psi_n(X_1) \\
        \vdots &\ddots &\vdots \\
       \psi_1(X_N)  & \ldots & \psi_n(X_N)
    \end{pmatrix}
     \begin{pmatrix}
        \psi_1(X_1) & \ldots & \psi_n(X_1) \\
        \vdots &\ddots &\vdots \\
       \psi_1(X_N)  & \ldots & \psi_n(X_N)
    \end{pmatrix}^{\top}.
    \]
    Since $\bfL = N^{-1}\bfD(\rho^{-1/2}) \mathbf{\Psi}\mathbf{\Psi}^\top \bfD(\rho^{-1/2})$, we deduce that 
    \[\bfL_{ij}=N^{-1}\rho(X_i)^{-1/2}K(X_i,X_j)\rho(X_j)^{-1/2}.\] 
    A similar argument applies to $\widehat{\bfL}$. The lemma follows. 
\end{proof}

\begin{lemma}\label{lm:bernstein}
We have
\[ \E \Big \| \frac{1}{N} \mathbf{\Psi}^{\top} \bfD(\rho^{-1})\mathbf{\Psi} - \bfI_n \Big \|^2_{\Frob} \le \frac{nK_{\max}}{N} \cdot\]
Moreover, for any $t\ge 0$, we have
    \[ \P \Big (\Big \| \frac{1}{N} \mathbf{\Psi}^{\top} \bfD(\rho^{-1})\mathbf{\Psi} - \bfI_n \Big \|_{\Frob} \ge t \Big ) \le (n^2+1)\exp\Big (-\frac{Nt^2}{2K_{\max}(n+ t/3)} \Big ) \cdot \]
    In particular,
    for every $\delta\in (0,1)$, with probability at least $1-\delta$ we have
 \[ \Big \| \frac{1}{N} \mathbf{\Psi}^{\top} \bfD(\rho^{-1})\mathbf{\Psi} - \bfI_n \Big \|^2_{\Frob} \le 4K_{\max}^2 \Big ( \frac{1}{9N^2} \log^2 \Big (\frac{n^2+1}{\delta}\Big )+ \frac{1}{N} \log \Big (\frac{n^2+1}{\delta}\Big )\Big )\cdot\]
\end{lemma}

\begin{proof}
We define the following random vectors 
\begin{equation*}
    Y_i = \Big (   \psi_1(X_i)/\sqrt{\rho(X_i)}, \ldots ,\psi_n(X_i)/\sqrt{\rho(X_i)}\Big )^{\top} , \quad 1\le i \le N.
\end{equation*}
Then it is clear that
    \begin{equation*}
        \frac{1}{N} \mathbf{\Psi}^{\top} \bfD(\rho^{-1})\mathbf{\Psi} = \frac{1}{N} \sum_{i=1}^N Y_iY_i^\top.
    \end{equation*}

    We have
  $ \frac{1}{N} \mathbf{\Psi}^{\top} \bfD(\rho^{-1})\mathbf{\Psi} - \bfI_n\ = \frac{1}{N}\sum_{i=1}^N(Y_iY_i^{\top} - \bfI_n)$. It is easy to see that
  \begin{eqnarray*}
      \|Y_iY_i^{\top} - \bfI_n\|^2_{\Frob} &=& \sum_{k=1}^n \Big (\frac{\psi_k(X_i)^2}{\rho(X_i)}-1 \Big )^2 
      + 2\sum_{1\le k <j\le n}\frac{\psi_k(X_i)^2}{\rho(X_i)}\frac{\psi_j(X_i)^2}{\rho(X_i)} \\
      &=& \Big (\sum_{k=1}^n \frac{\psi_k(X_i)^2}{\rho(X_i)} \Big )^2 
      - 2\sum_{k=1}^n \frac{\psi_k(X_i)^2}{\rho(X_i)} + n \\
      &=& \Big (\frac{K(X_i,X_i)}{\rho(X_i)}\Big )^2 - 2 \frac{K(X_i,X_i)}{\rho(X_i)} + n .
  \end{eqnarray*}
  This particularly implies $\|Y_iY_i^{\top}-\bfI_n\|_{\Frob} \le K_{\max}$ almost surely for every $1\le i \le N$. Moreover, we have
  \[ \mathbb{E} \|Y_iY_i^{\top} - \bfI_n\|^2_{\Frob} \le nK_{\max}.\]
   Since $Y_1,\ldots,Y_N$ are independent, we have
   \[ \E \Big \|  \frac{1}{N} \mathbf{\Psi}^{\top} \bfD(\rho^{-1})\mathbf{\Psi} - \bfI_n \Big \|_{\Frob}^2 = \frac{1}{N^2} \sum_{i=1}^N \mathbb{E} \|Y_iY_i^{\top} - \bfI_n\|^2_{\Frob} \le \frac{n K_{\max}}{N} \]
  as desired.

Now observe that
  $ \frac{1}{N} \mathbf{\Psi}^{\top} \bfD(\rho^{-1})\mathbf{\Psi} - \bfI_n\ = \frac{1}{N}\sum_{i=1}^N(Y_iY_i^{\top} - \bfI_n)$ and $Y_i$'s are independent. By treating $Y_iY_i^\top - \bfI_n$ as a vector in $\R^{n^2+1}$ and 
  applying Bernstein inequality for vector-valued random variables (\cite{tropp2015}, Theorem 1.6.2), we have for every $t\ge 0$
  \[ \P \Big (\Big \| \frac{1}{N} \mathbf{\Psi}^{\top} \bfD(\rho^{-1})\mathbf{\Psi} - \bfI_n \Big \|_{\Frob} \ge t \Big ) \le (n^2+1)\exp\Big (-\frac{Nt^2}{2K_{\max}(n+ t/3)} \Big ) \cdot\]
   Let 
   \[ \delta = (n^2+1)\exp\Big (-\frac{Nt^2}{2K_{\max}(n+ t/3)} \Big ) \cdot\]
   Solving for $t\ge 0$ we have
   \[ t = \frac{K_{\max}}{3N} \log \Big (\frac{n^2+1}{\delta}\Big ) + \sqrt{\frac{K_{\max}^2}{9N^2} \log^2 \Big (\frac{n^2+1}{\delta}\Big )+ \frac{2nK_{\max}}{N} \log \Big (\frac{n^2+1}{\delta}\Big )} \cdot\]
   Using $(a+b)^2 \le 2 a^2+2b^2$, we have
   \begin{eqnarray*}
        t^2 &\le&  2\frac{K_{\max}^2}{9N^2} \log^2 \Big (\frac{n^2+1}{\delta}\Big )+ 2\frac{K_{\max}^2}{9N^2} \log^2 \Big (\frac{n^2+1}{\delta}\Big )+ \frac{4nK_{\max}}{N} \log \Big (\frac{n^2+1}{\delta}\Big )  \\
         &=&  \frac{4K_{\max}^2}{9N^2} \log^2 \Big (\frac{n^2+1}{\delta}\Big )+ \frac{4nK_{\max}}{N} \log \Big (\frac{n^2+1}{\delta}\Big )\\
          &\le& 4K_{\max}^2 \Big ( \frac{1}{9N^2} \log^2 \Big (\frac{n^2+1}{\delta}\Big )+ \frac{1}{N} \log \Big (\frac{n^2+1}{\delta}\Big )\Big )
   \end{eqnarray*}
   where we used $K_{\max} \ge n$ in the last inequality.
   Therefore, with probability at least $1-\delta$ we have
   \[ \Big \| \frac{1}{N} \mathbf{\Psi}^{\top} \bfD(\rho^{-1})\mathbf{\Psi} - \bfI_n \Big \|^2_{\Frob} \le 4K_{\max}^2 \Big ( \frac{1}{9N^2} \log^2 \Big (\frac{n^2+1}{\delta}\Big )+ \frac{1}{N} \log \Big (\frac{n^2+1}{\delta}\Big )\Big )\cdot\]
\end{proof}

A direct consequence of Lemma \ref{lm:mcdiarmid} and Lemma \ref{lm:bernstein} is the following:
\begin{corollary} \label{cor:joint-prob}
    With probability  at least $1-\delta-\delta'$, we have 
    \begin{eqnarray*}
       |F-\E[F]| + \|f\|^2_{\infty} \Big \| \frac{1}{N} \mathbf{\Psi}^{\top} \bfD(\rho^{-1})\mathbf{\Psi} - \bfI_n \Big \|^2_{\Frob}\le  \|f\|^2_{\infty} K^2_{\max}\mathscr{E}(n,N,\delta,\delta')
    \end{eqnarray*}
    where $\mathscr{E}(n,N,\delta,\delta')$ defined as in \eqref{eq:def-Error}, namely
    \[ \mathscr{E}(n,N,\delta,\delta') = 4\sqrt{\frac{2\log (2/\delta')}{N}}+ \frac{4}{9N^2} \log^2 \Big (\frac{n^2+1}{\delta}\Big )+ \frac{4}{N} \log \Big (\frac{n^2+1}{\delta}\Big ) \cdot\]
\end{corollary}

\subsubsection{Proof of Theorem \ref{thm:disc-cont-true}}

\begin{lemma} \label{lm:asbound}
    Given the data set $\Xcal$, for every bounded test function $f: E \rightarrow \R$, we have
    \[  \var[\L_{\Scal_{\disc}}(f)|\Xcal]  \le  F(X_1,\ldots,X_N) + 4 \|f\|^2_{\infty} \Big \|
    \frac{1}{N} \mathbf{\Psi}^{\top} \bfD(\rho^{-1})\mathbf{\Psi} - \bfI_n \Big \|^2_{\Frob}  \]
    and
     \[\var[\L_{\Scal_{\disc}}(f)|\Xcal] \ge \frac{1}{4}F(X_1,\ldots,X_N) - 2\|f\|^2_{\infty} \Big \|
    \frac{1}{N} \mathbf{\Psi}^{\top} \bfD(\rho^{-1})\mathbf{\Psi} - \bfI_n \Big \|^2_{\Frob}.\]
\end{lemma}

\begin{proof}
By Lemma \ref{lm:n-dpp}, $\Scal_{\disc}$ is a DPP on $\Xcal$ with a projection kernel $\bfK=[\bfK_{ij}]$. Then we have 
\[  \var[\L_{\Scal_{\disc}}(f)|\Xcal] = \frac{1}{2} \sum_{i,j \in [N]} |f(X_i)-f(X_j)|^2 \bfK_{ij}^2 .\]
Using the inequality $b^2 \le 2a^2+2(a-b)^2$, we have
\[ \bfK_{ij}^2 \le 2 \bfL_{ij}^2 + 2 (\bfK_{ij} - \bfL_{ij})^2, \quad\forall i,j.\]
Thus
\begin{eqnarray*} 
    \var[\L_{\Scal_{\disc}}(f)|\Xcal] 
    &\le& \sum_{i,j \in [N]} |f(X_i)-f(X_j)|^2 \bfL_{ij}^2 + \sum_{i,j \in [N]} |f(X_i)-f(X_j)|^2 (\bfK_{ij} - \bfL_{ij})^2 \notag \\
     &\le&  \sum_{i,j \in [N]} |f(X_i)-f(X_j)|^2 \bfL_{ij}^2 + 4 \|f\|^2_{\infty} \|\bfL - \bfK\|^2_{\Frob} \notag \\
     &=& \frac{1}{N^2}\sum_{i,j \in [N]} |f(X_i)-f(X_j)|^2 \frac{K(X_i,X_j)^2}{\rho(X_i)\rho(X_j)} + 4 \|f\|^2_{\infty} \|\bfL - \bfK\|^2_{\Frob} \notag\\
     &=& F(X_1,\ldots,X_N) + 4 \|f\|^2_{\infty} \|\bfL - \bfK\|^2_{\Frob} .
\end{eqnarray*}
By Lemma \ref{lm:n-dpp}, $\bfK$ shares the same eigenvectors with $\bfL$. Thus
\[ \|\bfL-\bfK\|^2_{\Frob}= \sum_{\lambda(\bfL)>0} (\lambda(\bfL)-1)^2.\]
Moreover, $\bfL = N^{-1}\bfD(\rho^{-1/2}) \mathbf{\Psi} \mathbf{\Psi}^\top \bfD(\rho^{-1/2})$ implies nonzero eigenvalues of $\bfL$ are the same as nonzero eigenvalues of 
$N^{-1} \mathbf{\Psi}^{\top}\bfD(\rho^{-1})\mathbf{\Psi}$.
Hence
\[ \sum_{\lambda(\bfL)>0} (\lambda(\bfL)-1)^2 \le  
\Big \|
    \frac{1}{N} \mathbf{\Psi}^{\top} \bfD(\rho^{-1})\mathbf{\Psi} - \bfI_n \Big \|^2_{\Frob}.\]
Therefore
\begin{equation*}
     \var[\L_{\Scal_{\disc}}(f)|\Xcal]  \le F(X_1,\ldots,X_N) + 4 \|f\|_{\infty}^2 \Big \|
    \frac{1}{N} \mathbf{\Psi}^{\top} \bfD(\rho^{-1})\mathbf{\Psi} - \bfI_n \Big \|^2_{\Frob}.
\end{equation*}
Using the inequality $b^2 \ge \frac{1}{2}a^2-(a-b)^2$, we 
have
\[ \bfK_{ij}^2 \ge \frac{1}{2} \bfL_{ij}^2 - (\bfK_{ij} - \bfL_{ij})^2, \quad\forall i,j.\]
By a similar argument, we can show that  
\begin{equation*}
    \var[\L_{\Scal_{\disc}}(f)|\Xcal] \ge \frac{1}{4}F(X_1,\ldots,X_N) - 2\|f\|^2_{\infty} \Big \|
    \frac{1}{N} \mathbf{\Psi}^{\top} \bfD(\rho^{-1})\mathbf{\Psi} - \bfI_n \Big \|^2_{\Frob}.
\end{equation*}
\end{proof}

\begin{proposition}[Expectation bound] \label{prop:expected-bound-discrete-DPP}
     We have 
    \begin{eqnarray*}
     \E\Big[\var[\L_{\Scal_{\disc}}(f)|\Xcal] \Big] 
     &\le& 2\var[\L_{\Scal_{\cont}}(f)] + \frac{4\|f\|^2_{\infty}nK_{\max}}{N},  \quad \text{and}\\
     \E\Big[\var[\L_{\Scal_{\disc}}(f)|\Xcal] \Big]  &\ge& \frac{N-1}{2N}\var[\L_{\Scal_{\cont}}(f)] -\frac{2\|f\|^2_{\infty}nK_{\max}}{N}.
\end{eqnarray*}
\end{proposition}

\begin{proof}
     By Lemma \ref{lm:asbound}, we have 
     \[  \var[\L_{\Scal_{\disc}}(f)|\Xcal]  \le  F(X_1,\ldots,X_N) + 4 \|f\|^2_{\infty} \Big \|
    \frac{1}{N} \mathbf{\Psi}^{\top} \bfD(\rho^{-1})\mathbf{\Psi} - \bfI_n \Big \|^2_{\Frob}.  \]
    By Lemma \ref{lm:mcdiarmid}, we have 
    \[ \E [F(X_1,\ldots,X_N)] = \frac{2(N-1)}{N} \var [\L_{\Scal_{\cont}}(f)] \le 2 \var [\L_{\Scal_{\cont}}(f)].\]
    By Lemma \ref{lm:bernstein}, we have 
    \[ \E \Big \| \frac{1}{N} \mathbf{\Psi}^{\top} \bfD(\rho^{-1})\mathbf{\Psi} - \bfI_n \Big \|^2_{\Frob} \le \frac{nK_{\max}}{N} \cdot\]
    Combining these ingredients, we deduce that
    \[  \E\Big [\var[\L_{\Scal_{\disc}}(f)|\Xcal] \Big ] \le  2 \var [\L_{\Scal_{\cont}}(f)] + 4 \|f\|^2_{\infty} \frac{nK_{\max}}{N} \cdot  \]
    Similarly, we get the lower bound for the expectation.
\end{proof}

Now we are ready to complete the proof of Theorem \ref{thm:disc-cont-true}.

\begin{proof}[End of proof of Theorem \ref{thm:disc-cont-true}]
   By Lemma \ref{lm:asbound}
    \begin{eqnarray*}
         \var[\L_{\Scal_{\disc}}(f)|\Xcal]  &\le&  F(X_1,\ldots,X_N) + 4 \|f\|^2_{\infty} \Big \|
    \frac{1}{N} \mathbf{\Psi}^{\top} \bfD(\rho^{-1})\mathbf{\Psi} - \bfI_n \Big \|^2_{\Frob} \\
    &\le& \E[F] + |F-\E[F]| + 4 \|f\|^2_{\infty} \Big \|
    \frac{1}{N} \mathbf{\Psi}^{\top} \bfD(\rho^{-1})\mathbf{\Psi} - \bfI_n \Big \|^2_{\Frob} \\
    &\le& 2\var[\L_{\Scal_{\cont}}(f)] + 4\Big (|F-\E[F]|+ \|f\|^2_{\infty} \Big \|
    \frac{1}{N} \mathbf{\Psi}^{\top} \bfD(\rho^{-1})\mathbf{\Psi} - \bfI_n \Big \|^2_{\Frob} \Big )
    \end{eqnarray*}  
    and 
     \begin{eqnarray*}
         \var[\L_{\Scal_{\disc}}(f)|\Xcal]  &\ge&  \frac{1}{4}F(X_1,\ldots,X_N) - 2\|f\|^2_{\infty} \Big \|
    \frac{1}{N} \mathbf{\Psi}^{\top} \bfD(\rho^{-1})\mathbf{\Psi} - \bfI_n \Big \|^2_{\Frob} \\
    &\ge& \frac{1}{4}\E[F] - \frac{1}{4}|F-\E[F]| - 2 \|f\|^2_{\infty} \Big \|
    \frac{1}{N} \mathbf{\Psi}^{\top} \bfD(\rho^{-1})\mathbf{\Psi} - \bfI_n \Big \|^2_{\Frob} \\
    &\ge& \frac{N-1}{2N}\var[\L_{\Scal_{\cont}}(f)] - 2\Big (|F-\E[F]|+ \|f\|^2_{\infty} \Big \|
    \frac{1}{N} \mathbf{\Psi}^{\top} \bfD(\rho^{-1})\mathbf{\Psi} - \bfI_n \Big \|^2_{\Frob} \Big ).
    \end{eqnarray*}  
    Applying Corollary \ref{cor:joint-prob}, the theorem follows.
\end{proof}

\subsubsection{Proof of Theorem \ref{thm:disc-cont-approx}}

\begin{lemma} \label{lm:perturb-rho}
    If $\max_{1\le i \le N} \Big |\frac{\rho(X_i)}{\hat{\rho}(X_i)}-1 \Big| \le \varepsilon$, we have
    \[ \Big \|\frac{1}{N}\mathbf{\Psi}^{\top} \bfD(\hat\rho^{-1}) \mathbf{\Psi} - \bfI_n \Big \|_{\Frob} ^2  \le 2\varepsilon^2 n +8 \Big \|\frac{1}{N}\mathbf{\Psi}^{\top} \bfD(\rho^{-1}) \mathbf{\Psi} - \bfI_n \Big \|_{\Frob}^2.\]
\end{lemma}

\begin{proof}
    We have
\begin{eqnarray*}
   &&\Big \|\frac{1}{N}\mathbf{\Psi}^{\top} \bfD(\hat\rho^{-1}) \mathbf{\Psi} - \bfI_n \Big \|_{\Frob}\\
    &\le& \Big \|\frac{1}{N}\mathbf{\Psi}^{\top} \Big (\bfD(\hat\rho^{-1}) - \bfD(\rho^{-1})\Big )\mathbf{\Psi} \Big \|_{\Frob} 
    + \Big \|\frac{1}{N}\mathbf{\Psi}^{\top} \bfD(\rho^{-1}) \mathbf{\Psi} - \bfI_n \Big \|_{\Frob} \\
    &=& \frac{1}{N}\Big \|\mathbf{\Psi}^{\top}\bfD(\rho^{-1/2}) \Big (\bfD(\rho/\hat\rho) - \bfI_N\Big )\bfD(\rho^{-1/2})\mathbf{\Psi} \Big \|_{\Frob} 
    + \Big \|\frac{1}{N}\mathbf{\Psi}^{\top} \bfD(\rho^{-1}) \mathbf{\Psi} - \bfI_n \Big \|_{\Frob}.
\end{eqnarray*}

Now note that 
\begin{eqnarray*}
    \Big \|\mathbf{\Psi}^{\top}\bfD(\rho^{-1/2}) \Big (\bfD(\rho/\hat\rho) - \bfI_N\Big )\bfD(\rho^{-1/2})\mathbf{\Psi} \Big \|_{\Frob} \le \|\bfD(\rho/\hat\rho) - \bfI_N\|_{\op}\|\mathbf{\Psi}^{\top}\bfD(\rho^{-1})\mathbf{\Psi}\|_{\Frob}.
\end{eqnarray*}

Since $|\rho(X_i)/\hat\rho(X_i)-1|\le \varepsilon$ for every $i$, one has $\|\bfD(\rho/\hat\rho) - \bfI_N\|_{\op} \le \varepsilon$. Therefore
\begin{eqnarray*}
  \Big \|\frac{1}{N}\mathbf{\Psi}^{\top} \bfD(\hat\rho^{-1}) \mathbf{\Psi} - \bfI_n \Big \|_{\Frob}  &\le& \varepsilon \Big \|\frac{1}{N} \mathbf{\Psi}^{\top}\bfD(\rho^{-1})\mathbf{\Psi}\Big \|_{\Frob} +  \Big \|\frac{1}{N}\mathbf{\Psi}^{\top} \bfD(\rho^{-1}) \mathbf{\Psi} - \bfI_n \Big \|_{\Frob}\\
  &\le& \varepsilon \|\bfI_n\|_{\Frob} +  (1+\varepsilon)\Big \|\frac{1}{N}\mathbf{\Psi}^{\top} \bfD(\rho^{-1}) \mathbf{\Psi} - \bfI_n \Big \|_{\Frob} \\
  &=& \varepsilon \sqrt{n} +  (1+\varepsilon)\Big \|\frac{1}{N}\mathbf{\Psi}^{\top} \bfD(\rho^{-1}) \mathbf{\Psi} - \bfI_n \Big \|_{\Frob}\\
  &\le& \varepsilon \sqrt{n} +  2\Big \|\frac{1}{N}\mathbf{\Psi}^{\top} \bfD(\rho^{-1}) \mathbf{\Psi} - \bfI_n \Big \|_{\Frob},
\end{eqnarray*}
where we used the fact that $\varepsilon <1$.
Using the elementary inequality $(a+b)^2 \le 2a^2 + 2b^2$, we have 
\[ \Big \|\frac{1}{N}\mathbf{\Psi}^{\top} \bfD(\hat\rho^{-1}) \mathbf{\Psi} - \bfI_n \Big \|_{\Frob} ^2  \le 2\varepsilon^2 n +8 \Big \|\frac{1}{N}\mathbf{\Psi}^{\top} \bfD(\rho^{-1}) \mathbf{\Psi} - \bfI_n \Big \|_{\Frob}^2\]
as desired.
\end{proof}

\begin{lemma} \label{lm:asbound-approx}
     If $\max_{1\le i \le N} \Big |\frac{\rho(X_i)}{\hat{\rho}(X_i)}-1 \Big| \le \varepsilon$, we have
    \[\var[\L_{\widehat{\Scal}_{\disc}}(f)|\Xcal] \le (1+\varepsilon)^2F(X_1,\ldots,X_N)+ 32 \|f\|^2_{\infty}  \Big \|\frac{1}{N}\mathbf{\Psi}^{\top} \bfD(\rho^{-1}) \mathbf{\Psi} - \bfI_n \Big \|^2_{\Frob} + 8 \|f\|^2_{\infty}\varepsilon^2n, \]
and 
\[\var[\L_{\widehat{\Scal}_{\disc}}(f)|\Xcal] \ge \frac{(1-\varepsilon)^2}{4}F(X_1,\ldots,X_N) - 16 \|f\|^2_{\infty}  \Big \|\frac{1}{N}\mathbf{\Psi}^{\top} \bfD(\rho^{-1}) \mathbf{\Psi} - \bfI_n \Big \|^2_{\Frob} - 4 \|f\|^2_{\infty}\varepsilon^2n. \]
\end{lemma}

\begin{proof}
By similar arguments as in Lemma \ref{lm:asbound}, we have 
\[ \var[\L_{\widehat{\Scal}_{\disc}}(f)|\Xcal] \le \frac{1}{N^2} \sum_{i,j \in [N]} |f(X_i)-f(X_j)|^2 \frac{|K(X_i,X_j)|^2}{\hat\rho(X_i)\hat\rho(X_j)} + 4 \|f\|^2_{\infty} \Big \|\frac{1}{N}\mathbf{\Psi}^{\top} \bfD(\hat\rho^{-1}) \mathbf{\Psi} - \bfI_n \Big \|^2_{\Frob},\]
and 
\[ \var[\L_{\widehat{\Scal}_{\disc}}(f)|\Xcal] \ge \frac{1}{4N^2} \sum_{i,j \in [N]} |f(X_i)-f(X_j)|^2 \frac{|K(X_i,X_j)|^2}{\hat\rho(X_i)\hat\rho(X_j)} - 2 \|f\|^2_{\infty} \Big \|\frac{1}{N}\mathbf{\Psi}^{\top} \bfD(\hat\rho^{-1}) \mathbf{\Psi} - \bfI_n \Big \|^2_{\Frob}.\]
Since $1-\varepsilon\le |\rho(X_i)/\hat\rho(X_i)| \le 1 + \varepsilon$ for every $i$, we have
\[ (1-\varepsilon)^2\frac{|K(X_i,X_j)|^2}{\rho(X_i)\rho(X_j)}\le  \frac{|K(X_i,X_j)|^2}{\hat\rho(X_i)\hat\rho(X_j)} \le  (1+\varepsilon)^2\frac{|K(X_i,X_j)|^2}{\rho(X_i)\rho(X_j)}, \quad \forall \: i,j.\]
Therefore
\[\var[\L_{\widehat{\Scal}_{\disc}}(f)|\Xcal] \le (1+\varepsilon)^2F(X_1,\ldots,X_N)+ 4 \|f\|^2_{\infty} \Big \|\frac{1}{N}\mathbf{\Psi}^{\top} \bfD(\hat\rho^{-1}) \mathbf{\Psi} - \bfI_n \Big \|^2_{\Frob}, \]
and 
\[\var[\L_{\widehat{\Scal}_{\disc}}(f)|\Xcal] \ge \frac{(1-\varepsilon)^2}{4}F(X_1,\ldots,X_N) - 2 \|f\|^2_{\infty} \Big \|\frac{1}{N}\mathbf{\Psi}^{\top} \bfD(\hat\rho^{-1}) \mathbf{\Psi} - \bfI_n \Big \|^2_{\Frob}. \]
By Lemma \ref{lm:perturb-rho}, we deduce that
\[\var[\L_{\widehat{\Scal}_{\disc}}(f)|\Xcal] \le (1+\varepsilon)^2F(X_1,\ldots,X_N)+ 32 \|f\|^2_{\infty}  \Big \|\frac{1}{N}\mathbf{\Psi}^{\top} \bfD(\rho^{-1}) \mathbf{\Psi} - \bfI_n \Big \|^2_{\Frob} + 8 \|f\|^2_{\infty}\varepsilon^2n, \]
and 
\[\var[\L_{\widehat{\Scal}_{\disc}}(f)|\Xcal] \ge \frac{(1-\varepsilon)^2}{4}F(X_1,\ldots,X_N) - 16 \|f\|^2_{\infty} \Big \|\frac{1}{N}\mathbf{\Psi}^{\top} \bfD(\rho^{-1}) \mathbf{\Psi} - \bfI_n \Big \|^2_{\Frob} - 4 \|f\|^2_{\infty}\varepsilon^2n \]
as desired.
\end{proof}

\begin{proof}[End of proof of Theorem \ref{thm:disc-cont-approx}]
 If $\max_{1\le i \le N} \Big |\frac{\rho(X_i)}{\hat{\rho}(X_i)}-1 \Big| \le \varepsilon$, by Lemma \ref{lm:asbound-approx}
\begin{eqnarray*}
     &&\var[\L_{\widehat{\Scal}_{\disc}}(f)|\Xcal]\\ 
     &\le& (1+\varepsilon)^2 \E[F] + (1+\varepsilon)^2|F - \E[F]|+ 32 \|f\|^2_{\infty}  \Big \|\frac{1}{N}\mathbf{\Psi}^{\top} \bfD(\rho^{-1}) \mathbf{\Psi} - \bfI_n \Big \|^2_{\Frob} + 8 \|f\|^2_{\infty}\varepsilon^2n \\
     &\le& (1+\varepsilon)^22 \var[\L_{\Scal_{\cont}}(f)] + 32 \Big (|F-\E[F]|+  \|f\|^2_{\infty}  \Big \|\frac{1}{N}\mathbf{\Psi}^{\top} \bfD(\rho^{-1}) \mathbf{\Psi} - \bfI_n \Big \|^2_{\Frob}\Big )  + 8 \|f\|^2_{\infty}\varepsilon^2n 
\end{eqnarray*}
and 
\begin{eqnarray*}
    &&\var[\L_{\widehat{\Scal}_{\disc}}(f)|\Xcal] \\
    &\ge& \frac{(1-\varepsilon)^2}{4}\E[F] - \frac{(1-\varepsilon)^2}{4}|F-\E[F]|- 16 \|f\|^2_{\infty} \Big \|\frac{1}{N}\mathbf{\Psi}^{\top} \bfD(\rho^{-1}) \mathbf{\Psi} - \bfI_n \Big \|^2_{\Frob} - 4 \|f\|^2_{\infty}\varepsilon^2n \\
    &\ge& {(1-\varepsilon)^2}\frac{N-1}{2N}\var[\L_{\Scal_{\cont}}(f)] - 16 \Big (|F-\E[F]|+  \|f\|^2_{\infty}  \Big \|\frac{1}{N}\mathbf{\Psi}^{\top} \bfD(\rho^{-1}) \mathbf{\Psi} - \bfI_n \Big \|^2_{\Frob}\Big ) - 4 \|f\|^2_{\infty}\varepsilon^2n .
\end{eqnarray*}
Applying Corollary \ref{cor:joint-prob} completes the proof.
\end{proof}

\section{Technical Appendices to wavelet DPPs} \label{a:appendix-wavelet}
In this section, we provide detailed discussions and proofs for our theorems in Section \ref{sec:waveletDPP}. We first recall a useful formula to compute variance of linear statistics of DPPs that will be used throughout this section.

\begin{lemma} \label{lm:var-DPP-formula}
Let $\Scal$ be a DPP on a Polish space $(E,\mu)$ whose kernel $K$ has the following orthonormal expansion
\[ K(x,y) = \sum_{k=1}^n \phi_k(x)\overline{\phi_k(y)}, \quad x,y \in E,\]
where $\{\phi_k:1\le k \le n\}$ is an orthonormal set in $L^2(\mu):=L^2(E,\mu)$. Then for any test function $f:E\rightarrow \R$, we have 
\[ \var[\L_{\Scal}(f)] = \sum_{k=1}^n \|\proj_{H^\perp}(f\phi_k)\|^2_{L^2(\mu)}\]
where $H^{\perp}$ is the orthogonal complement of the subspace $H=\span\{\phi_k:1\le k \le n\}$ in $L^2(\mu)$.
\end{lemma}

\begin{proof}
    We have
    \begin{eqnarray*}
        \var[\L_{\Scal}(f)] &=& \int |f(x)|^2K(x,x)\d\mu(x) - \iint f(x)\overline{f(y)}|K(x,y)|^2 \d\mu(x)\d\mu(y) \\
        &=& \int |f(x)|^2\sum_{k=1}^n |\phi_k(x)|^2 \d \mu(x) - \iint f(x)\overline{f(y)}\Big |\sum_{k=1}^n\phi_k(x)\overline{\phi_k(y)}\Big |^2\d\mu(x) \d\mu(y) \\
        &=&\sum_{k=1}^n \| f\phi_k\|^2_{L^2(\mu)} - \sum_{k,k'=1}^n |\langle f\phi_k,\phi_{k'}\rangle_{L^2(\mu)} |^2\\
        &=& \sum_{k=1}^n \Big ( \|f\phi_k\|^2_{L^2(\mu)} - \sum_{k'=1}^n |\langle f\phi_k, \phi_{k'}\rangle_{L^2(\mu)}|^2 \Big ) =  \sum_{k=1}^n \|\proj_{H^\perp}(f\phi_k)\|^2_{L^2(\mu)}.
    \end{eqnarray*}
\end{proof}

\subsection{Brief introduction to wavelets}

We briefly recall some standard facts from wavelet theory, following the notation of \cite[Chapter 7]{Mal08}. 
\begin{definition}
    A multiresolution approximation of \(L^2(\mathbb{R})\) is a sequence of closed subspaces
\((V_j)_{j\in\mathbb{Z}}\) satisfying
\[
\cdots \subset V_2 \subset V_1 \subset V_0 \subset V_{-1}\subset V_{-2}\subset \cdots,
\]
together with
\[
\overline{\bigcup_{j\in\mathbb{Z}} V_j}=L^2(\mathbb{R}),
\qquad
\bigcap_{j\in\mathbb{Z}} V_j=\{0\},
\]
and the scaling relation
\[
f(x)\in V_j
\quad\Longleftrightarrow\quad
f(2x)\in V_{j-1}.
\]
\end{definition}

The space \(V_j\) should be understood as the approximation space at scale \(2^j\); thus, smaller values of \(j\) correspond to finer resolution.

\begin{definition}
    Given a multiresolution approximation $(V_j)_{j\in \Z}$ of $L^2(\R)$, a function \(\phi\in L^2(\mathbb{R})\) is called a scaling function if its integer translates \(
\{\phi(\cdot-k):k\in\mathbb{Z}\}
\) generate the reference space \(V_0\). In addition, $\phi$ is called an orthonormal scaling function if 
\(
\{\phi(\cdot-k):k\in\mathbb{Z}\}
\)
is an orthonormal basis of \(V_0\).
\end{definition}
 The corresponding normalized dilates and translates are
\[
\phi_{j,k}(x)
=
2^{-j/2}\phi(2^{-j}x-k),
\qquad j\in\mathbb{Z},\ k\in\mathbb{Z}.
\]
Then, for each \(j\in\mathbb{Z}\),
\[
V_j
=
\overline{\operatorname{span}}\{\phi_{j,k}:k\in\mathbb{Z}\}.
\]
The functions \(\phi_{j,k}\) therefore describe the low-frequency or approximation component of a function at scale \(2^j\).

\begin{lemma} \label{lm:scaling-function}
Let \(\phi\in L^\infty(\mathbb R)\) be a compactly supported normalized
orthonormal scaling function associated with a multiresolution analysis, normalized so that
$\int_{\mathbb R}\phi(x)\d x=1$.
Then
\[
    \sum_{k\in\mathbb Z}\phi(x-k)=1
    \qquad\text{for a.e. }x\in\R.
\]
Consequently, 
\[\sum_{\bfk\in\mathbb Z^d}2^{-dj/2}\Phi_{-j,\bfk}(\bfx)=1 \quad \text{for a.e. }\bfx\in\R^d.\] Moreover, there exists $c_{\phi}>0$ such that $\sum_{k\in\mathbb Z}\phi(x-k)^2 \ge c_\phi$ for a.e. $x\in \R$.
\end{lemma}

\begin{proof}
We use the Fourier transform convention
   \( \widehat f(\xi)
    =
    \int_{\mathbb R} f(x)e^{-i\xi x}\d x .
\)
Since \(\phi\) is compactly supported and belongs to \(L^\infty(\mathbb R)\),
we have \(\phi\in L^1(\mathbb R)\cap L^2(\mathbb R)\), and hence
\(\widehat\phi\) is continuous. The normalization gives
\[
    \widehat\phi(0)
    =
    \int_{\mathbb R}\phi(x)\d x
    =
    1.
\]
Because the integer translates
\(
    \{\phi(\cdot-k):k\in\mathbb Z\}
\)
are orthonormal in \(L^2(\mathbb R)\), the standard Fourier characterization
of orthonormal translates gives
\[
    \sum_{\ell\in\mathbb Z}
    \bigl|\widehat\phi(\xi+2\pi \ell)\bigr|^2
    =
    1
    \qquad\text{for a.e. }\xi\in\mathbb R.
\]
We claim that
\[
    \widehat\phi(2\pi m)=0 \quad \text{ for all }m\in\mathbb Z\setminus\{0\}.
\]
Indeed, suppose that for some \(m\ne 0\) we have
\(
    \bigl|\widehat\phi(2\pi m)\bigr|>0,
\) then by continuity of \(\widehat\phi\), for \(\xi\) sufficiently close to \(0\),
both
\(
    \bigl|\widehat\phi(\xi)\bigr|
\)
is close to \(1\), and
\(
    \bigl|\widehat\phi(\xi+2\pi m)\bigr|
\)
is bounded away from \(0\). Hence, on a set of positive measure near
\(\xi=0\), we would have
\(
    \sum_{\ell\in\mathbb Z}
    \bigl|\widehat\phi(\xi+2\pi \ell)\bigr|^2
    >
    1,
\)
contradicting the orthonormal-translate identity. The claim follows.

We define the periodization
   $ P_{\phi}(x)
    :=
    \sum_{k\in\mathbb Z}\phi(x-k)$.
Since \(\phi\) is compactly supported, this sum is locally finite. Hence
\(P_\phi\) is a well-defined \(1\)-periodic function in \(L^2_{\mathrm{loc}}\).
For \(m\in\mathbb Z\), its \(m\)-th Fourier coefficient is
\[ \int_0^1 P_\phi(x)e^{-2\pi i m x}\d x =  \sum_{k\in\mathbb Z}
    \int_0^1 \phi(x-k)e^{-2\pi i m x}\d x  =
    \int_{\mathbb R}\phi(u)e^{-2\pi i m u}\d u  =
    \widehat\phi(2\pi m). \]
Hence,
\(
    \int_0^1 P_\phi(x)\d x = \widehat\phi(0)=1,
\)
while for \(m \ne 0\), we have \(
    \int_0^1 P_\phi(x)e^{-2\pi i m x}\d x
    =
    \widehat\phi(2\pi m)
    =
    0.
\)
Thus \(P_\phi\) has the same Fourier coefficients as the constant function \(1\).
Hence
\(P_\phi(x)=1\)
    for a.e. \(x\in\mathbb R,
\)
that is,
\[
    \sum_{k\in\mathbb Z}\phi(x-k)=1
    \qquad\text{for a.e. }x\in\mathbb R.
\]
Next, recall that
\(
    \phi_{-j,k}(x)
    =
    2^{j/2}\phi(2^j x-k),
\)
which implies that for \(\bfx=(x_1,\ldots,x_d)\) and \(\bfk=(k_1,\ldots,k_d)\), we have
\[
    2^{-dj/2}\Phi_{-j,\bfk}(x)
    =
    \prod_{r=1}^d
    \phi(2^j x_r-k_r).
\]
Summing over \(\bfk\in\mathbb Z^d\), we get
\[
    \sum_{\bfk\in\mathbb Z^d}
    2^{-dj/2}\Phi_{-j,\bfk}(\bfx)
    =
    \sum_{k_1,\ldots,k_d\in\mathbb Z}
    \prod_{r=1}^d
    \phi(2^j x_r-k_r) =
    \prod_{r=1}^d
    \left(
        \sum_{k_r\in\mathbb Z}
        \phi(2^j x_r-k_r)
    \right) =
1
\]
for a.e. \(\bfx\in\mathbb R^d\) as desired.

Finally, since \(\phi\) is compactly supported, there exists an integer
\(M_\phi<\infty\) such that, for every \(x\in\mathbb R\), at most
\(M_\phi\) terms in the sum
\(
    \sum_{k\in\mathbb Z}\phi(x-k)
\)
are nonzero. By Cauchy-Schwarz, we have
\[
\begin{aligned}
    1
    =
    \Big (
        \sum_{k\in\mathbb Z}\phi(x-k)
    \Big)^2 \le
    \left(
        \#\{k\in\mathbb Z:\phi(x-k)\ne 0\}
    \right)
    \sum_{k\in\mathbb Z}\phi(x-k)^2 \le
    M_\phi
    \sum_{k\in\mathbb Z}\phi(x-k)^2
\end{aligned}
\]
for a.e. \(x\in\mathbb R\). Let $c_{\phi}:= M_\phi^{-1}$, we then have
\(
    \sum_{k\in\mathbb Z}\phi(x-k)^2
    \ge
    c_\phi>0\)
for a.e. \(x\in\mathbb R\). This completes the proof.
\end{proof}

\subsection{Spaces of test functions} \label{sec:test-function}

In this section, we present the definitions of H\"older continuous functions and Sobolev spaces that we mentioned in Section \ref{sec:waveletDPP}.

\begin{definition}[H\"older continuous functions]
    Let $s\in (0,1]$. The space of $s$-H\"older continuous functions on $\mathbb{R}^d$, denoted by $C^{0,s}$, consist of all functions $f:\R^d \rightarrow \R$ such that
    \[ |f|_{C^{0,s}} := \sup_{\bfx\ne \bfy} \frac{|f(\bfx) - f(\bfy)|}{\|\bfx - \bfy\|^s} < \infty.\]
\end{definition}

\begin{remark}
    If $f$ satisfies \eqref{eq:cpt-supp}, namely, $f(\bfx) = 0$ on $\R^d \setminus (0,1)^d$, then we have
    \[|f|_{C^{0,s}}  = \sup_{\bfx,\bfy \in [0,1]^d,\bfx\ne \bfy} \frac{|f(\bfx) - f(\bfy)|}{\|\bfx - \bfy\|^s} \cdot \]
\end{remark}
 
We also consider the Sobolev spaces $H^s(\Omega)$, where $\Omega$ is an open domain in $\R^d$. 
For $s\in \mathbb{N}$, the common definition of Sobolev spaces is as follows:

\begin{definition}[Sobolev spaces]
    Let $s\in \N, s\ge 1$, the Sobolev space $H^s(\Omega)$ is defined by
    \[ H^s(\Omega) := \Big \{ f\in L^2(\Omega): D^{\alpha}f \in L^2(\Omega) \text{ for every multi-index } \alpha \text{ with } |\alpha| \le s\Big \}\]
    where $D^{\alpha}f$ denotes the weak derivative of $f$. The Sobolev norm is defined by
    \[ \|f\|^2_{H^s(\Omega)} := \sum_{|\alpha|\le s}  \|D^{\alpha} f\|^2_{L^2(\Omega)}.\]
    In particular, for $s=1$ 
    \[\|f\|^2_{H^1(\Omega)} = \|f\|^2_{L^2(\Omega)} + \|\nabla f\|^2_{L^2(\Omega)}.\]
\end{definition}

 For $s<1$, we take the following definition for the fractional Sobolev space $H^s(\Omega)$:

\begin{definition}[Fractional Sobolev spaces]
    Let $s\in (0,1)$, the fractional Sobolev space $H^s(\Omega)$ is defined by
   \begin{equation*}
       H^{s}(\Omega):= \Big \{ f\in L^2(\Omega) : |f|_{H^{s}(\Omega)} < \infty  \Big \},
   \end{equation*}
   where
  \begin{equation*}
       |f|_{H^{s}(\Omega)}:= \Big(\iint_{\Omega \times \Omega}\frac{(f(\bfx)-f(\bfy))^2}{\|\bfx-\bfy\|^{d+2s}} \d\bfx \d\bfy\Big)^{1/2},
   \end{equation*}
   called the Gagliardo seminorm of $f$.
\end{definition}

We are particularly interested in the case $\Omega = (0,1)^d$. For simplicity of notation, we set 
\[H^s:=H^s((0,1)^d).\]

\subsection{Proof of Theorem \ref{thm:wavelet-dpp}}

In this section, we present the detailed proof of Theorem \ref{thm:wavelet-dpp}. Let $\phi$ be a bounded, compactly supported, orthonormal scaling function associated with a multiresolution analysis, and normalized by $\int \phi(x)\d x =1$. We consider the wavelet DPP kernel $K$ defined in \eqref{eq:wavelet-dpp-kernel}, namely, 
    \[ K(\bfx,\bfy) = \sum_{\bfk\in \Ical} \Phi_{-j,\bfk}(\bfx)\Phi_{-j,\bfk}(\bfy).\]
    Let $H:= \span\{\Phi_{-j,\bfk}:\bfk\in \Ical\}$ in $L^2(\R^d, \d\bfx)$. 
    Let $\Scal$ be the DPP on $\R^d$ defined by $K$ with respect to the Lebesgue measure $\d \bfx$.
    Let $f$ be a test function satisfying \eqref{eq:cpt-supp}, then by Lemma \ref{lm:var-DPP-formula} we have 
    \[ \var[\L_{\Scal}(f)] = \sum_{\bfk \in \Ical} \|\proj_{H^\perp} (f\Phi_{-j,\bfk})\|^2_{L^2(\d\bfx)}.\]

    \begin{remark}
            In what follows, $C$ will denote a constant depending only on $d$ and $\phi$, which possibly changes from line by line.
    \end{remark}

\subsubsection{The case of $C^{0,s}$.}
We first prove for the case $f\in C^{0,s}$ and $f$ satisfies \eqref{eq:cpt-supp}. In particular, we will show that 
\[ \var[\L_{\Scal}(f)] \le C |f|^2_{C^{0,s}} n^{1-2s/d},\]
where $C>0$ depending only on $\phi$ and $d$. 

To this end, for each $\bfk\in \Ical$, we fix a point $\bfx_\bfk \in \supp(\Phi_{-j,\bfk})$. We define a remainder function 
    \[R_\bfk(\bfx):= f(\bfx) - f(\bfx_\bfk).\] 
    Then
    \[ \proj_{H^\perp}(f\Phi_{-j,\bfk}) = f(\bfx_\bfk)\proj_{H^\perp} (\Phi_{-j,\bfk}) +\proj_{H^\perp}(R_\bfk\Phi_{-j,\bfk}) = \proj_{H^\perp}(R_\bfk\Phi_{-j,\bfk})\]
    since $\Phi_{-j,\bfk} \in H$. 
    Therefore
    \begin{eqnarray*}
        \|\proj_{H^\perp} (f\Phi_{-j,\bfk})\|^2_{L^2(\d\bfx)} 
        &=& \|\proj_{H^\perp} (R_\bfk\Phi_{-j,\bfk})\|^2_{L^2(\d\bfx)} \le  \|R_\bfk\Phi_{-j,\bfk}\|^2_{L^2(\d\bfx)} \\
        &=& \int_{\supp(\Phi_{-j,\bfk})} |f(\bfx)-f(\bfx_\bfk)|^2 |\Phi_{-j,\bfk}(\bfx)|^2 \d \bfx \\
        &\le& \sup_{\bfx \in \supp(\Phi_{-j,\bfk})}|f(\bfx)-f(\bfx_\bfk)|^2
    \end{eqnarray*}
    since $\int |\Phi_{-j,\bfk}(\bfx)|^2\d\bfx = 1$. 
    Note that $f\in C^{0,s}$, we have
   \[ |f(\bfx)-f(\bfx_\bfk)|^2 \le |f|^2_{C^{0,s}} \|\bfx - \bfx_\bfk\|^{2s}, \quad \forall \bfx. \]
    This implies 
    \begin{eqnarray*}
        \|\proj_{H^\perp} (f\Phi_{-j,\bfk})\|^2_{L^2(\d\bfx)}  &\le& |f|^2_{C^{0,s}}  \sup_{\bfx \in \supp(\Phi_{-j,\bfk})} \|\bfx - \bfx_\bfk\|^{2s} \\
        &\le& |f|^2_{C^{0,s}} \diam(\supp(\Phi_{-j,\bfk}))^{2s},
    \end{eqnarray*}
    where $\diam(A):= \sup_{\bfx,\bfy \in A} \|\bfx - \bfy\|$ for $A\subset \R^d$.
    Note that $\Phi_{-j,\bfk}(\bfx) = \phi_{-j,k_1}(x_1)\ldots \phi_{-j,k_d}(x_d)$, which implies 
    \[ \diam(\supp(\Phi_{-j,\bfk}))= \sqrt{d}\cdot 2^{-j} \cdot \diam(\supp(\phi)).\]
    Thus
     \[  \|\proj_{H^\perp} (f\Phi_{-j,\bfk})\|^2_{L^2(\d\bfx)}  \le \diam(\supp(\phi))^{2s} |f|^2_{C^{0,s}}   d^{s} 2^{-2s j}, \quad \forall \bfk \in \Ical.  \]
     Therefore,
     \[\var[\L_{\Scal}(f)] \le n \cdot \diam(\supp(\phi))^{2s} |f|^2_{C^{0,s}}   d^{s} 2^{-2s j} \le C(\phi,d) \cdot |f|^2_{C^{0,s}} n^{1-2s/d},\]
     as desired.
    
\subsubsection{The case of $H^s$ with $s<1$.}
Let $f\in H^s = H^s((0,1)^d)$ with $s\in (0,1)$. We now show that 
\[ \var[\L_{\Scal}(f)] \le C |f|^2_{H^s} n^{1-2s/d},\]
where $C>0$ depending only on $\phi$ and $d$. 
To this end,
let $\mathrm{Conv}[A]$ denote the convex hull of a set $A$.
     For each $\bfk\in \Ical$, let $S_{\bfk} := \mathrm{Conv}[\supp(\Phi_{-j,\bfk}) ]\cap (0,1)^d$, and define
     \[ \bar{f}_{\bfk}:= \frac{1}{\vol(S_\bfk)}\int_{S_\bfk} f(\bfy)\d\bfy. \]
     By the same argument as above, we have
     \begin{eqnarray*}
          &&\|\proj_{H^\perp} (f\Phi_{-j,\bfk})\|^2_{L^2(\d\bfx)} =  \|\proj_{H^\perp}( (f-\bar{f}_{\bfk})\Phi_{-j,\bfk} ) \|^2_{L^2(\d\bfx)} \\
          &\le& \int_{S_\bfk} |f(\bfx) - \bar{f}_{\bfk}|^2 |\Phi_{-j,\bfk}(\bfx)|^2 \d \bfx\\
          &\le& 2^{dj} \|\phi\|_{\infty}^{2d}  \int_{S_\bfk} \Big |f(\bfx) -  \frac{1}{\vol(S_\bfk)}\int_{S_\bfk} f(\bfy)\d\bfy \Big |^2 \d \bfx.
     \end{eqnarray*}
    By the fractional Poincar\'e-Wirtinger inequality (see e.g. \cite{acosta2017}, Proposition 2.2), we have for $f\in H^s((0,1)^d)$
    \[\int_{S_\bfk} \Big |f(\bfx) -  \frac{1}{\vol(S_\bfk)}\int_{S_\bfk} f(\bfy)\d\bfy \Big |^2\d\bfx\le C \cdot \diam(S_\bfk)^{2s} |f|^2_{H^s(S_\bfk)} \]
    where $C$ depends only on $\phi$ and $d$, and 
    \[|f|^2_{H^s(\Omega)}:= \iint_{\Omega\times \Omega}\frac{(f(\bfx)-f(\bfy))^2}{\|\bfx-\bfy\|^{d+2s}} \d\bfx \d\bfy\]
    for a domain $\Omega \subset \R^d$. Now note that 
    \[\diam(S_\bfk) \le \diam(\supp(\Phi_{-j,\bfk})) = \diam(\supp(\phi)) \cdot 2^{-j}\sqrt{d}.\]
    Thus
    \[\int_{S_\bfk} \Big |f(\bfx) -  \frac{1}{\vol(S_\bfk)}\int_{S_\bfk} f(\bfy)\d\bfy \Big |^2\d\bfx\le C \cdot 2^{-2sj} |f|^2_{H^s(S_\bfk)} \]
    for some constant $C$ depending only on $d$ and $\phi$. Hence
    \[ \var[\L_{\Scal}(f)] \le C \cdot 2^{dj-2sj} \sum_{\bfk\in \Ical} |f|^2_{H^s(S_\bfk)}.\]
    
    We observe that
    \begin{eqnarray*}
        \sum_{\bfk\in \Ical}|f|^2_{H^s(S_\bfk)} &=&\iint_{(0,1)^d\times(0,1)^d}\frac{(f(\bfx)-f(\bfy))^2}{\|\bfx-\bfy\|^{d+2s}}\Big (\sum_{\bfk\in \Ical} \mathbf{1}_{S_\bfk}(\bfx)\mathbf{1}_{S_\bfk}(\bfy) \Big )\d\bfx \d\bfy \\
        &\le& \iint_{(0,1)^d\times(0,1)^d}\frac{(f(\bfx)-f(\bfy))^2}{\|\bfx-\bfy\|^{d+2s}}\Big (\sum_{\bfk\in \Ical} \mathbf{1}_{S_\bfk}(\bfx) \Big )
        \Big (\sum_{\bfk\in \Ical}\mathbf{1}_{S_\bfk}(\bfy) \Big )\d\bfx \d\bfy.
    \end{eqnarray*}
    For $\bfk, \bfk'\in \Ical$, we have $S_\bfk \cap S_{\bfk'} \ne \varnothing$ only if 
    \[\mathrm{Conv}[\supp(\phi_{-j,k_i})] \cap \mathrm{Conv}[\supp(\phi_{-j,k_i'})] \ne \varnothing \quad\text{ for all }1\le i \le d.\]
   Assuming that $\supp(\phi) \subset [a,b]$, we have
   \[ \mathrm{Conv}[\supp(\phi_{-j,k})]\subset [2^{-j}(a+k), 2^{-j}(b+k)]\]
   which implies that 
   \[ \#\{k' \in \Z: \mathrm{Conv}[\supp(\phi_{-j,k})] \cap \mathrm{Conv}[\supp(\phi_{-j,k'})] \ne \varnothing\} \le 2(b-a)+1.\]
   Thus, for every $\bfx \in (0,1)^d$, we have 
   \[ \sum_{\bfk \in \Ical} \mathbf{1}_{S_\bfk}(\bfx) \le (2(b-a)+1)^d,\]
   which implies 
   \[ \sum_{\bfk\in \Ical}|f|^2_{H^s(S_\bfk)}  \le (2(b-a)+1)^{2d} |f|^2_{H^s((0,1)^d)}.\]
   Combining all ingredients, we deduce that 
   \[\var[\L_{\Scal}(f)] \le C|f|^2_{H^s((0,1)^d)}  n^{1-2s/d} \]
   where $C$ depending only on $d$ and $\phi$ as desired.

\subsubsection{The case of $H^1$.}

Let $f\in H^1=H^1((0,1)^d)$.  We now show that 
\[ \var[\L_{\Scal}(f)] \le C \|\nabla f\|^2_{L^2((0,1)^d)} n^{1-2s/d},\]
where $C>0$ depending only on $\phi$ and $d$. The proof follows exactly as the case $H^s$ with $s<1$ above, where the Poincar\'e-Wirtinger inequality replaced by the Payne–Weinberger Poincaré inequality \cite{payne60}, which states that for every bounded, convex open domain $\Omega \subset \R^d$
\[ \int_{\Omega} \Big |f(\bfx) - \frac{1}{\vol(\Omega)}\int_{\Omega} f(\bfy) \d\bfy \Big |^2 \d \bfx \le \frac{\diam(\Omega)^2}{\pi^2} \|\nabla f\|^2_{L^2(\Omega)}, \quad \forall f\in H^1(\Omega).\]

\subsection{Proof of Proposition~\ref{p:stratified_sampling}}
Consider the kernel $\eqref{eq:wavelet-dpp-kernel}$, for some $j\geq 1$ and $\phi=\phi_{\mathrm{Haar}}$ the Haar scaling function. Note that in that case, the rank of the kernel is $n=2^{dj}$, that $K(x,x)=n$ for all $x\in[0,1]^d$, and that for $k\neq \ell$, 
$$
    \mathrm{supp}(\mathbf{\Phi}_{-j,k}) \cap \mathrm{supp}(\mathbf{\Phi}_{-j,\ell}) = \emptyset.
$$
Actually, the supports of $\mathbf{\Phi_{-j,k}}$, $k\in\mathcal{I}$ are a partition of $[0,1]^d$.
Moreover, by construction \eqref{eq:wavelet-dpp-kernel}, for $k\in\mathcal{I}$, 
$$
    x,y\in\mathrm{supp}(\mathbf{\Phi}_{-j,k}) \Rightarrow K(x,y) = n,
$$ 
while $K(x,y)=0$ as soon as $x$ and $y$ belong to the supports of two different wavelets.
Now consider the pdf \eqref{e:density_dpp} defining DPP($K$,$\d x$). If any two points $x_i,x_j$, $i\neq j$ are in the support of the same $\mathbf{\Phi}_{-j,k}$, then two columns of the matrix in the determinant are equal, and the determinant is thus zero. Meanwhile, the RHS of \eqref{e:density_dpp} is constant on configurations $x_1, \dots, x_n$ that belong to $n$ different supports.
This proves the claim.

\section{Technical Appendices to applications}
\label{a:appendix-application}
\subsection{A quadrature for the Haar DPP}

The simple linear statistic \eqref{e:first_MC_estimator} has fast-decaying mean squared error when the scaling function $\phi$ is the Haar scaling function.
In light of Proposition~\ref{p:stratified_sampling}, the result is expected since the rate had been known for stratified sampling since \citep{haber1966modified}.
For completeness, we still state the proposition and prove it as a consquence of our more general result for wavelet-based DPPs.

\begin{proposition}[Haar wavelet DPPs as quadratures] \label{p:wavelet-dpp-nor}
    Let $\Scal$ be the DPP  on $(\R^d, \d\bfx)$ defined by the kernel $K$ in \eqref{eq:wavelet-dpp-kernel} with $\phi_{\mathrm{Haar}}$ the Haar scaling function. 
    Let $f$ be a function satisfying \eqref{eq:cpt-supp}. 
    Then, there exists a constants $C>0$ depending only on $d$ and $\omega$ such that 
        \[\E \Big [ \Big|\Lcal_{\Scal}(f) - \int f(\bfx) \omega(\bfx) \d \bfx) \Big |^2\Big ] \le C \|f\|^2_s  n^{-1-2s/d},\]
       where 
       \[ \|f\|^2_s:= 
       \begin{cases}
           \|f\|^2_{\infty} + |f|^2_{C^{0,s}} & \text{ if } f\in C^{0,s} \\
            \|f\|^2_{L^2((0,1)^d)} + |f|^2_{H^s} & \text{ if } f\in H^s, s<1 \\ 
             \|f\|^2_{L^2((0,1)^d)} + \|\nabla f\|^2_{L^2((0,1)^d)} & \text{ if } f\in H^s, s=1. \\ 
       \end{cases}
       \]
\end{proposition}

\begin{proof}
    Note that for the Haar scaling function $\phi=\phi_{\mathrm{Haar}}$, we have for a.e. $\bfx \in (0,1)^d$
    \[ K(\bfx,\bfx) = \sum_{\bfk\in \Ical} \Phi_{-j,\bfk}(\bfx)^2 =  n.\]
    Let $f\in C^{0,s}$ for $s \in (0,1]$. Since $\omega \in C^1$, we have
    \begin{eqnarray*}
        |f(\bfx)\omega(\bfx) -f(\bfy)\omega(\bfy)| &\le& \|\omega\|_{\infty} |f(\bfx)-f(\bfy)| + \|f\|_{\infty}|\omega(\bfx)-\omega(\bfy)| \\
        &\le& \|\omega\|_{\infty} |f|_{C^{0,s}} \|\bfx - \bfy\|^{s} + \|f\|_{\infty} \|\nabla \omega\|_{\infty} \|\bfx - \bfy\| \\
        &\le& \Big (\|\omega\|_{\infty} |f|_{C^{0,s}} + d^{(1-s)/2}\|f\|_{\infty} \|\nabla \omega\|_{\infty} \Big )\|\bfx - \bfy\|^{s}.
    \end{eqnarray*}
    This implies $f\omega \in C^{0,s}$, and 
    \[ |f\omega|^2_{C^{0,s}} \le C (|f|_{C^{0,s}} ^2 + \|f\|^2_{\infty} ) = C\|f\|_s^2\]
    for some constant $C>0$ depending only on $d$ and $\omega$.
    Applying Theorem \ref{thm:wavelet-dpp} yields
    \begin{eqnarray*}
       \E \Big [ \Big|\Lcal_{\Scal}(f) - \int f(\bfx) \omega(\bfx) \d \bfx) \Big |^2\Big ]
       = \var \Big [\L_{\Scal}\Big (\frac{f\omega}{n}\Big )\Big ] 
       \lesssim  \|f\|_s^2 n^{-1-2s/d}.
    \end{eqnarray*}
    
    By a similar argument, we also can show that if $f\in H^s$ and $\omega \in C^1$, then $f\omega \in H^s$ with
    \[    |f\omega|_{H^s}^2 \le  C \|f\|^2_s\]
    for some constant $C>0$ depending only on $d$ and $\omega$. Applying Theorem \ref{thm:wavelet-dpp} yields the desired result.
\end{proof}

\subsection{Proof of Theorem \ref{thm:wavelet-dpp-adj}}
\begin{remark}
    In what follows, we use the notations $\gtrsim, 
\lesssim$ to denote inequalities up to multiplicative constants which depend only on $d$ and $\phi$.
\end{remark}

\begin{proof}
   Let \(\{\bfx_{\bfk}:\bfk\in\Ical\}\) be the nodes designed as in \eqref{eq:adjusted-linear-stat}, i.e., be such that for each \(\bfk\in\Ical\), \(\bfx_{\bfk}\) is an arbitrarily chosen but fixed point in \(\supp(\Phi_{-j,\bfk})\).
   For a function $f$ with $\supp(f)\Subset (0,1)^d$, we define the following operator $Q$ acting on $f$ as
\begin{equation}
Qf(\bfx) := \frac{1}{2^{dj/2}} \sum_{\bfk \in \Ical} \Phi_{-j,\bfk}(\bfx)  f(\bfx_\bfk).
\end{equation}
We then have
\[\int Qf(\bfx) \d \bfx = \frac{1}{2^{dj/2}} \sum_{\bfk\in \Ical} f(\bfx_\bfk) \int \Phi_{-j,\bfk}(\bfx) \d \bfx = \frac{1}{2^{dj}} \sum_{\bfk \in \Ical} f(\bfx_\bfk),\]
where we used the fact that 
\[ \int \Phi_{-j,\bfk}(\bfx) \d \bfx = \prod_{i=1}^d \Big (\int \phi_{-j,k_i}(x_i)\d x_i\Big ) = \prod_{i=1}^d \Big (\int 2^{j/2}\phi(2^jx_i-k_i)\d x_i\Big ) =2^{-dj/2}.\]
We now show that:
\begin{equation} \label{eq:var-Qf}
    \var\Big [ \L_{\Scal}\Big ( \frac{f-Qf}{K}\Big )\Big ] = O(n^{-1-2s/d})
\end{equation}
for all function $f\in C^{0,s}$ with $\supp(f)\Subset (0,1)^d$.

To this end, we first observe that
    \[  \var\Big [ \L_{\Scal}\Big ( \frac{f-Qf}{K}\Big )\Big ] 
    \le \int \Big |\frac{f(\bfx)-Qf(\bfx)}{K(\bfx,\bfx)}\Big |^2 K(\bfx,\bfx) \d\bfx.\]
    Since $\supp(f)$ is bounded away from the boundary of the hypercube $[0,1]^d$, there exists $\delta>0$ such that $\supp(f) \subset (\delta,1-\delta)^d$. Note that for each $\delta>0$, there exists $j_0$ such that for $j\ge j_0$ we have
    \[ K(\bfx,\bfx) = 2^{dj}\prod_{i=1}^d \Big (\sum_{k\in \Z} \phi(2^jx_i-k)^2 \Big ), \quad \forall \bfx \in (\delta,1-\delta)^d.\]
    By Lemma \ref{lm:scaling-function}, there exists $c_{\phi}>0$ so that $\sum_{k\in \Z}\phi(x-k)^2 \ge c_{\phi}>0$ for a.e. $x\in \R$. Thus we have $K(\bfx,\bfx) \gtrsim n $ for all $j\ge j_0$ and a.e. $\bfx \in (\delta,1-\delta)^d$.
    Thus, for $j\ge j_0$ we would have 
    \[\var\Big [ \L_{\Scal}\Big ( \frac{f-Qf}{K}\Big )\Big ]  \lesssim \frac{1}{n} \int_{(\delta,1-\delta)^d}|f(\bfx)-Qf(\bfx)|^2 \d\bfx.\] 

    To control $|f(\bfx)-Qf(\bfx)|$, we note that 
    \[ \sum_{\bfk\in \Z^d} 2^{-dj/2}\Phi_{-j,\bfk}(\bfx) = 1, \quad \text{ for a.e. } \bfx \in \R^d.\]
    This implies 
    \[ f(\bfx) =  \sum_{\bfk\in \Z^d} 2^{-dj/2}\Phi_{-j,\bfk}(\bfx)  f(\bfx), \quad  \text{ for a.e. } \bfx \in \R^d.\]
    Note that if $\supp(f) \subset (\delta,1-\delta)^d$, for large enough $j$ (depending on $\supp(\phi)$ and $\delta$) we would have $(\delta,1-\delta)^d \cap \supp(\Phi_{-j,\bfk}) = \varnothing$ if $\bfk \notin \Ical$. Hence, for large enough $j$, we would have
    \[ f(\bfx) =  \sum_{\bfk\in\Ical} 2^{-dj/2}\Phi_{-j,\bfk}(\bfx)  f(\bfx), \quad  \text{ for a.e. } \bfx \in (\delta,1-\delta)^d.\]
    This implies that for a.e. $\bfx \in (\delta,1-\delta)^d$
    \begin{eqnarray*}
       |f(\bfx)-Qf(\bfx)| &\le& 2^{-dj/2} \sum_{\bfk\in \Ical} |\Phi_{-j,\bfk}(\bfx)||f(\bfx) - f(\bfx_\bfk)| \\
       &\le&\|\phi\|^d_{\infty} \sum_{\bfk \in \Ical} \mathbf{1}_{S_{\bfk}}(\bfx)|f(\bfx) - f(\bfx_\bfk)|,
    \end{eqnarray*}
   where $S_{\bfk}:= \supp(\Phi_{-j,\bfk}) \cap (0,1)^d$. 
   For each $\bfk \in \Ical$, we have 
   \[ |f(\bfx) - f(\bfx_\bfk)| \lesssim |f|_{C^{0,s}}2^{-s j}.\]
    By the proof of Theorem \ref{thm:wavelet-dpp} (ii), the sum $\sum_{\bfk \in \Ical} \mathbf{1}_{S_{\bfk}}(\bfx)$ is finite and depends only on $d$ and $\phi$. Therefore, for a.e. $\bfx \in (\delta,1-\delta)^d$
   \[ |f(\bfx) - Qf(\bfx)|^2 \lesssim |f|^2_{C^{0,s}}2^{-2s j}. \]
Thus,
   \[\var\Big [ \L_{\Scal}\Big ( \frac{f-Qf}{K}\Big )\Big ] \le C |f|^2_{C^{0,s}} n^{-1-2s/d} \]
   for some constant $C>0$ depending only on $\phi$ and $d$. The claim \eqref{eq:var-Qf} follows.

  Now, we observe that 
   \begin{eqnarray*}
       \widetilde{\Lcal}_{\Scal}(f) 
&=&  \Lcal_{\Scal}(f) 
- \frac{1}{2^{dj/2}}\sum_{i=1}^n \sum_{\bfk \in \Ical} \frac{\Phi_{-j,\bfk}(Y_i)}{K(Y_i,Y_i)}  f(\bfx_\bfk)\omega(\bfx_\bfk)
+ \frac{1}{2^{dj}}\sum_{\bfk \in \Ical} f(\bfx_\bfk)\omega(\bfx_\bfk)\\
&=& \L_{\Scal}\Big (\frac{f\omega}{K}\Big ) - \sum_{i=1}^n \frac{Q(f\omega)(Y_i)}{K(Y_i,Y_i)} + \int Q(f\omega)(\bfx) \d \bfx \\
&=& \L_{\Scal}\Big (\frac{f\omega}{K}\Big ) - \L_{\Scal}\Big (\frac{Q(f\omega)}{K}\Big ) + \int Q(f\omega)(\bfx) \d \bfx \\
&=& \L_{\Scal}\Big (\frac{f\omega -Q(f\omega)}{K}\Big ) + \int Q(f\omega)(\bfx) \d \bfx .
   \end{eqnarray*}
   Taking expectation gives
   \[ \E[ \widetilde{\Lcal}_{\Scal}(f) ] = \int \frac{f(\bfx)\omega(\bfx) - Q(f\omega)(\bfx)}{K(\bfx,\bfx)} K(\bfx,\bfx)\d\bfx + \int Q(f\omega)(\bfx)\d\bfx = \int f(\bfx)\omega (\bfx) \d \bfx.\]
   This shows $\widetilde{\Lcal}_{\Scal}(f)$ is an unbiased estimator of $\int f(\bfx)\omega (\bfx) \d \bfx$. 

   Finally, notice that if $f\in C^{0,s}$ and $\omega\in C^1$, then $f\omega \in C^{0,s}$ with
     \[ |f\omega|^2_{C^{0,s}} \le C \Big (|f|_{C^{0,s}} ^2 + \|f\|^2_{\infty} \Big )\]
    for some constant $C>0$ depending only on $d$ and $\omega$.
   Moreover, we have $\supp(f\omega) \subset \supp(\omega) \Subset (0,1)^d$. Applying the claim \eqref{eq:var-Qf} for the function $f\omega$ yields the desired result.
\end{proof}

\subsection{Proof of Theorem \ref{thm:discrete-wavelet-coreset}}

\begin{proof}[Proof of Theorem \ref{thm:discrete-wavelet-coreset}]
Notice that 
\[\E[L_{\Scal_{\disc}}(f)|\Xcal ] = \sum_{\bfx\in \Xcal} \frac{f(\bfx)}{\P( \bfx\in \Scal_{\disc}|\Xcal)}\P( \bfx\in \Scal_{\disc}|\Xcal) = \sum_{\bfx\in \Xcal}f(\bfx) = L(f),\] 
which implies $\E[|L_{\Scal_{\disc}}(f) - L(f)|^2|\Xcal] = \var[L_{\Scal_{\disc}}(f)|\Xcal]$. 
We define 
\[ \widetilde{L}_{\Scal_{\disc}}(f):= \sum_{X_i \in \Scal_{\disc}} \frac{f(X_i)}{\bfL_{ii}} = \sum_{\bfx\in \Scal_{\disc}} \frac{Nf(\bfx)\rho(\bfx)}{K(\bfx,\bfx)} = N\L_{\Scal_{\disc}}\Big(\frac{f\rho}{K}\Big ) \cdot \]
Then, we have
\begin{eqnarray} \label{eq:dwc-1}
    \E \Big [\Big |\frac{L_{\Scal_{\disc}}(f) - L(f)}{N} \Big |^2 \Big |\Xcal\Big ] &=&\var\Big [\frac{1}{N}L_{\Scal_{\disc}}(f) \Big |\Xcal \Big ] \notag\\
    &\le& 2 \var\Big [\frac{1}{N}\widetilde{L}_{\Scal_{\disc}}(f) \Big |\Xcal \Big] + 2 \var\Big[\frac{L_{\Scal_{\disc}}(f) - \widetilde{L}_{\Scal_{\disc}}(f) }{N} \Big | \Xcal \Big ] \notag \\
    &=& 2 \var\Big [\L_{\Scal_{\disc}}\Big(\frac{f\rho}{K}\Big ) \Big |\Xcal \Big] + 2 \var\Big[\frac{L_{\Scal_{\disc}}(f) - \widetilde{L}_{\Scal_{\disc}}(f) }{N} \Big | \Xcal \Big ].
\end{eqnarray}

We note that $K(\bfx,\bfy)$ is built by the Haar scaling function, which yields $K(\bfx,\bfx) = n$ for a.e. $\bfx\in (0,1)^d$. This implies 
\[ \sup_{\bfx\in \supp(\rho)} \frac{K(\bfx,\bfx)}{\rho(\bfx)} \le \frac{n}{\rho_{\min}}\cdot\]
Thus we can choose $K_{\max} = n/\rho_{\min}$.
By Theorem \ref{thm:disc-cont-true}, there is an event $\mathcal{A}$ with $\P(\mathcal{A}) \ge 1-\delta-\delta'$ such that on $\mathcal{A}$ we have the following inequality
\[ \var\Big [\L_{\Scal_{\disc}}\Big(\frac{f\rho}{K}\Big ) \Big |\Xcal \Big] \le  2\var\Big [\L_{\Scal}\Big(\frac{f\rho}{K}\Big ) \Big] + 4\Big \|\frac{f\rho}{K} \Big\|^2_{\infty}  \frac{n^2}{\rho_{\min}^2}\mathscr{E}(n,N,\delta,\delta') \]
where $\mathscr{E}(n,N,\delta,\delta')$ is defined as \eqref{eq:def-Error}. Note that $\|f\rho/K\|_{\infty} = \|f\rho\|_{\infty}/n \le \|f\|_{\infty} \|\rho\|_{\infty}/n$. Thus
\begin{eqnarray} \label{eq:dwc-2}
    \var\Big [\L_{\Scal_{\disc}}\Big(\frac{f\rho}{K}\Big ) \Big |\Xcal \Big] \le  2\var\Big [\L_{\Scal}\Big(\frac{f\rho}{K}\Big ) \Big] + 4  \Big (\frac{\|f\|_{\infty}\|\rho\|_{\infty}}{\rho_{\min}}\Big )^2\mathscr{E}(n,N,\delta,\delta'). 
\end{eqnarray}
 Now observe that if $f\in C^{0,s}$ and $\rho\in C^1$, we have $f\rho \in C^{0,s}$. By Theorem \ref{thm:wavelet-dpp}, there exists $C>0$ depending only on $d$ and $\phi$ such that
\begin{equation} \label{eq:dwc-3}
     \var\Big [\L_{\Scal}\Big(\frac{f\rho}{K}\Big ) \Big] \le C |f\rho|^2_{C^{0,s}} n^{-1-2s/d}.
\end{equation}
By \eqref{eq:dwc-2} and \eqref{eq:dwc-3}, we deduce that on the event $\mathcal{A}$
\begin{equation} \label{eq:dwc-4}
     \var\Big [\L_{\Scal_{\disc}}\Big(\frac{f\rho}{K}\Big ) \Big |\Xcal \Big] \le 2C |f\rho|^2_{C^{0,s}} 
     n^{-1-2s/d} + 4  \Big (\frac{\|f\|_{\infty}\|\rho\|_{\infty}}{\rho_{\min}}\Big )^2\mathscr{E}(n,N,\delta,\delta') .
\end{equation}

\medskip

On the other hand, we recall that by Lemma \ref{lm:n-dpp}, for each realization of the data set $\Xcal$, we have $\Scal_{\disc}$ is a discrete DPP on $\Xcal$ defined by some projection matrix $\bfK$.
Then $\P(X_i\in \Scal_{\disc}|\Xcal) = \bfK_{ii}$ and thus
\[L_{\Scal_{\disc}}(f) = \sum_{X_i \in \Scal_{\disc}} \frac{f(X_i)}{\bfK_{ii}} \cdot\]
Hence
\begin{eqnarray*}
    \var\Big[\frac{L_{\Scal_{\disc}}(f) - \widetilde{L}_{\Scal_{\disc}}(f) }{N} \Big | \Xcal \Big ] &=& \frac{1}{N^2} \var \Big [\sum_{X_i \in \Scal_{\disc}} f(X_i) \Big (\frac{\bfL_{ii}-\bfK_{ii}}{\bfK_{ii}\bfL_{ii}}\Big ) \Big | \Xcal\Big ] \\
    &\le& \frac{1}{N^2} \sum_{i=1}^N f(X_i)^2  \Big (\frac{\bfL_{ii}-\bfK_{ii}}{\bfK_{ii}\bfL_{ii}}\Big )^2 \bfK_{ii}.
\end{eqnarray*}
By Lemma \ref{lm:n-dpp}, if we write $\lambda_1(\bfL) \ge \lambda_2(\bfL) \ge \ldots \ge \lambda_N(\bfL)$ for the eigenvalues of $\bfL$, then 
\[ |\bfL_{ii} - \bfK_{ii}| \le \Big (\max_{1\le j \le n} |\lambda_j(\bfL)-1| \Big ) \bfK_{ii}, \quad \forall i.\]
Thus 
\begin{eqnarray*}
 \var\Big[\frac{L_{\Scal_{\disc}}(f) - \widetilde{L}_{\Scal_{\disc}}(f) }{N} \Big | \Xcal \Big ] \le \frac{\|f\|^2_{\infty}}{N^2}\Big (\max_{1\le j \le n} |\lambda_j(\bfL)-1|^2 \Big ) \sum_{i=1}^N \frac{\bfK_{ii}}{\bfL_{ii}^2}   \cdot
\end{eqnarray*}
We note that 
\[ \bfL_{ii} = \frac{K(X_i,X_i)}{N\rho(X_i)}=\frac{n}{N \rho(X_i)} \quad\text{almost surely}\]
since $K$ is built by the Haar scaling function. 
By Lemma \ref{lm:L-structure}, we have 
\[\max_{1\le j\le n}|\lambda_j(\bfL)-1|^2 = \|N^{-1}\Phi^\top \bfD(\rho^{-1})\Phi - \bfI_n\|^2_{\op} \le  
\|N^{-1}\Phi^\top \bfD(\rho^{-1})\Phi - \bfI_n\|^2_{\Frob}. \]
Note that on $\mathcal{A}$ we have 
\[\|N^{-1}\Phi^\top \bfD(\rho^{-1})\Phi - \bfI_n\|^2_{\Frob} \le  \frac{n^2}{\rho_{\min}^2} \Big ( \frac{4}{9N^2} \log^2 \Big (\frac{n^2+1}{\delta}\Big )+ \frac{4}{N} \log \Big (\frac{n^2+1}{\delta}\Big )\Big ).\]
Combining all ingredients with $\sum_{i=1}^N \bfK_{ii} \le n$, we deduce that on $\mathcal{A}$
\begin{equation} \label{eq:dwc-5}
    \var\Big[\frac{L_{\Scal_{\disc}}(f) - \widetilde{L}_{\Scal_{\disc}}(f) }{N} \Big | \Xcal \Big ] 
\le 4 \Big (\frac{\|f\|_{\infty}\|\rho\|_{\infty}}{\rho_{\min}}\Big )^2 \Big ( \frac{n}{9N^2} \log^2 \Big (\frac{n^2+1}{\delta}\Big )+ \frac{n}{N} \log \Big (\frac{n^2+1}{\delta}\Big )\Big ).
\end{equation}
Combining \eqref{eq:dwc-1}, \eqref{eq:dwc-4} and \eqref{eq:dwc-5}, on the event $\mathcal{A}$, we have
\begin{eqnarray*}
    \var\Big [\frac{1}{N}L_{\Scal_{\disc}}(f) \Big |\Xcal \Big ] 
    &\le& 4 C |f\rho|^2_{C^{0,s}} n^{-1-2s/d} + 8 \Big (\frac{\|f\|_{\infty}\|\rho\|_{\infty}}{\rho_{\min}}\Big )^2\mathscr{E}(n,N,\delta,\delta')\\
    &+& 8 \Big (\frac{\|f\|_{\infty}\|\rho\|_{\infty}}{\rho_{\min}}\Big )^2 \Big ( \frac{n}{9N^2} \log^2 \Big (\frac{n^2+1}{\delta}\Big )+ \frac{n}{N} \log \Big (\frac{n^2+1}{\delta}\Big )\Big )\\
    &=& 4 C |f\rho|^2_{C^{0,s}} n^{-1-2s/d} + 8 \Big (\frac{\|f\|_{\infty}\|\rho\|_{\infty}}{\rho_{\min}}\Big )^2\widetilde{\mathscr{E}}(n,N,\delta,\delta')
\end{eqnarray*}
where
\begin{eqnarray*}
    \widetilde{\mathscr{E}}(n,N,\delta,\delta') &:=& \mathscr{E}(n,N,\delta,\delta') + \frac{n}{9N^2} \log^2 \Big (\frac{n^2+1}{\delta}\Big )+ \frac{n}{N} \log \Big (\frac{n^2+1}{\delta}\Big ) \\
    &=& 4\sqrt{\frac{2\log (2/\delta')}{N}} + \frac{n+4}{9N^2} \log^2 \Big (\frac{n^2+1}{\delta}\Big )+ \frac{n+4}{N} \log \Big (\frac{n^2+1}{\delta}\Big )  \cdot
\end{eqnarray*}
Therefore, for $N$ large enough such that $ \widetilde{\mathscr{E}}(n,N,\delta,\delta') \le n^{-1-2s/d}$, with probability at least $1-\delta-\delta'$ we have
\[\E \Big [\Big |\frac{L_{\Scal_{\disc}}(f) - L(f)}{N} \Big |^2 \Big |\Xcal\Big ] =\var\Big [\frac{1}{N}L_{\Scal_{\disc}}(f) \Big |\Xcal \Big ] \le C n^{-1-2s/d}\]
for some $C$ depends only on $d$ and $\rho$ as desired. Here we used the fact that $\|f\|_{\infty} \le 1$ and $|f|_{C^{0,s}} \le 1$ to decouple the dependence on $f$.
\end{proof}

\subsection{Coreset guarantees for discrete wavelet DPPs} \label{app:coreset-guarantee-wavelet}

\begin{theorem}[Discrete wavelet DPPs as coresets] \label{thm:discrete-wavelet-coreset-guarantee}
  Let $\Fcal$ be a class of test functions such that 
 \begin{equation} \label{eq:F1}
  \|f\|_{\infty} \le 1, |f|_{C^{0,\alpha}([0,1]^d)} \le 1 \text{ and } |L(f)| \ge cN \text{ for some } c>0.    \tag{F1}
 \end{equation} 
  We assume either
 \begin{equation} \label{eq:F2}
 \dim \span_{\R}(\Fcal) = \d_{\Fcal} <\infty, \tag{F2} \end{equation}
 or there exists a bounded subset $\Theta \subset \R^{\d_{\Fcal}}$ such that
 \begin{equation} \label{eq:F3}
     \Fcal=\{f_{\theta}:\theta\in \Theta\} \text{ and } \|f_{\theta}-f_{\theta'}\|_{\infty} \le \ell \|\theta-\theta'\|,\forall \theta,\theta'\in \Theta. \tag{F3}
 \end{equation}
 Then there exist $C_0,C_1>0$ depending only on $c$, $d$, and $\rho$ such that for all $\delta, \delta' \in (0,1)$, with probability at least $1-\delta-\delta'$ in the data set $\Xcal$, we have for $0\le \varepsilon\le C_0 n^{-2\alpha/d}$:
 \begin{itemize}
     \item[(i)]  under \eqref{eq:F1} and \eqref{eq:F2}:
   \[ \P\Big (\exists f\in \Fcal: \Big |\frac{L_{\Scal_{\disc}}(f)}{L(f)} -1 
   \Big | \ge \varepsilon  \Big | \Xcal\Big ) \le 2 \exp \Big (6\d_{\Fcal} - C_1 n^{1+2\alpha/d} \varepsilon^2\Big ),\]
   \item[(ii)] under \eqref{eq:F1} and \eqref{eq:F3}:
   \[ \P\Big (\exists f\in \Fcal: \Big |\frac{L_{\Scal_{\disc}}(f)}{L(f)} -1 
   \Big | \ge \varepsilon  \Big | \Xcal\Big ) \le 2 \exp \Big ((C_2 - \log \varepsilon)\d_{\Fcal} - C_1 n^{1+2\alpha/d} \varepsilon^2\Big )\]
   for some $C_2>0$ depends only on $\Theta,\ell,\rho$ and $c$,
 \end{itemize}
 provided that $\widetilde{\mathscr{E}}(n,N,\delta,\delta') \le n^{-1-2\alpha/d}$ and $n^{2\alpha/d}\rho_{\min}^2 \ge 8$.
\end{theorem}

\begin{proof}[Proof of Theorem \ref{thm:discrete-wavelet-coreset-guarantee}]
    Let us fix $\delta,\delta' \ge 0$ such that $\delta+\delta' <1$. 
    Let $N$ be large enough such that $\widetilde{\mathscr{E}}(n,N,\delta,\delta') \le n^{-1-2\alpha/d}$. By Theorem \ref{thm:discrete-wavelet-coreset}, there exists an event $\mathcal{A}$ (measurable w.r.t. the data set $\Xcal$) with probability at least $1-\delta-\delta'$, we have 
    \[ \var\Big [\frac{L_{\Scal_{\disc}}(f)}{N} \Big |\Xcal \Big ] \le C n^{-1-2\alpha/d}, \quad \forall \: f\in C^{0,\alpha}([0,1]^d): \|f\|_{\infty}\le 1,|f|_{C^{0,\alpha}([0,1]^d)}\le 1.\]
    Following the proof of Theorem \ref{thm:discrete-wavelet-coreset}, on the same event $\mathcal{A}$ we have 
    \begin{eqnarray*}
    \|N^{-1}\Phi^\top \bfD(\rho^{-1})\Phi - \bfI_n\|^2_{\Frob} &\le&  \frac{n^2}{\rho_{\min}^2} \Big ( \frac{4}{9N^2} \log^2 \Big (\frac{n^2+1}{\delta}\Big )+ \frac{4}{N} \log \Big (\frac{n^2+1}{\delta}\Big )\Big ) \\
    &\le& \frac{4n}{\rho_{\min}^2} \widetilde{\mathscr{E}}(n,N,\delta,\delta') \le \frac{4n^{-2\alpha/d}}{\rho_{\min}^2} \le \frac{1}{2}
    \end{eqnarray*}
    since $n^{2\alpha/d} \rho_{\min}^2 \ge 8$ by our assumption. This implies that $\rank(\Phi) = n$, and thus $|\Scal_{\disc}| =  n$ on this event $\mathcal{A}$. 
    
    Moreover, following the proof of Theorem \ref{thm:discrete-wavelet-coreset}, we have for every $i \in \{1,2,\ldots,N\}$
    \[\bfK_{ii} \ge \bfL_{ii} - |\bfK_{ii}-\bfL_{ii}| \ge \bfL_{ii} - \|N^{-1}\Phi^\top \bfD(\rho^{-1})\Phi - \bfI_n\|_{\Frob} \bfK_{ii} .\]
    Thus, on event $\mathcal{A}$ we have $\bfL_{ii} \le 2 \bfK_{ii}$ for every $i$. Now note that 
    \[\bfL_{ii} = \frac{n}{N \rho(X_i)} \ge \|\rho\|^{-1}_{\infty} \frac{n}{N} \quad \text{ almost surely.}\]
    Therefore, on this event $\mathcal{A}$, we have 
    \[ \bfK_{ii} \ge \frac{1}{2\|\rho\|_{\infty}} \frac{n}{N} \cdot\]
    Applying Theorem \ref{thm:coreset-dpp} with $V:= C n^{-1-2\alpha/d}$ completes the proof.
\end{proof}

\begin{remark}
    The parameter $\varepsilon$ should be interpreted as the relative error tolerance of the coreset approximation: smaller values of $\varepsilon$ correspond to higher accuracy. 
    For example, for $d\ge 4$, we have $-1/2 \le -2\alpha/d $; hence we can choose $\varepsilon=O(n^{-1/2-\alpha'/d})$ for any $\alpha'<\alpha$. Theorem \ref{thm:discrete-wavelet-coreset-guarantee} then says that the discrete wavelet DPP construction provides a high-probability coreset guarantee with accuracy rate of order $n^{-1/2-\alpha'/d}$, for any $\alpha'<\alpha$. Thus, when $\alpha =1$, we can achieve the accuracy rate close to $n^{-1/2-1/d}$, which is clearly better than the rate $n^{-1/2}$ of independent sampling, and the rate $n^{-1/2-1/(2d)}$ of discrete OPE-based DPPs (see Remark 6.1 in \cite{coreset-dpp}).
\end{remark}

For reader's convenience, we cite here Theorem 4 in \cite{coreset-dpp}, after harmonizing notations.
\begin{theorem}[Theorem 4 in \cite{coreset-dpp}] \label{thm:coreset-dpp}
Let $\Scal$ be a DPP with a Hermitian kernel $\bfK$ on a finite set $\Xcal = \{x_1,\ldots,x_N\}$ with $\E[|\Scal|]=n$. Assuming that there exists $\eta>0$ independent of $n,N$ such that:
\[ \bfK_{ii} \ge \eta\frac{n}{N}, \quad \forall 1\le i \le N.\]
 Let $\Fcal$ be a class of bounded functions on $\Xcal$ such that 
 \begin{equation} \label{eq:C1}
  \|f\|_{\infty} \le 1, |L(f)| \ge cN \text{ for some } c>0 \text{ uniformly on $\Fcal$} \tag{C1}
 \end{equation} 
  and either
 \begin{equation} \label{eq:C2}
 \dim \span_{\R}(\Fcal) = \d_{\Fcal} <\infty, \tag{C2} \end{equation}
 or there exists a bounded subset $\Theta \subset \R^{\d_{\Fcal}}$ such that
 \begin{equation} \label{eq:C3}
     \Fcal=\{f_{\theta}:\theta\in \Theta\} \text{ and } \|f_{\theta}-f_{\theta'}\|_{\infty} \le \ell \|\theta-\theta'\|,\forall \theta,\theta'\in \Theta. \tag{C3}
 \end{equation}
 Let $L_{\Scal}(f):= \sum_{x\in \Scal} f(x)/\P(x\in \Scal)$, and let
 \[V \ge \sup_{f\in \Fcal} \var\Big [\frac{L_{\Scal}(f)}{N}\Big ].\]
 Then there exists a universal constant $A>0$ such that for $0\le \varepsilon\le \frac{4A\eta}{3c}nV$:
 \begin{itemize}
     \item[(i)] under \eqref{eq:C1} and \eqref{eq:C2} we have 
    \[ \P \Big( \exists f \in \Fcal: \Big |\frac{L_{\Scal}(f)}{L(f)} -1 \Big |\ge \varepsilon \Big ) \le 2 \exp \Big (6\d_{\Fcal} - \frac{c^2\varepsilon^2}{16AV} \Big ),\]
    \item[(ii)] under \eqref{eq:C1} and \eqref{eq:C3}, if $|\Scal|\le Bn$, we have 
    \[ \P \Big( \exists f \in \Fcal: \Big |\frac{L_{\Scal}(f)}{L(f)} -1 \Big |\ge \varepsilon \Big ) \le 2 \exp \Big (C\d_{\Fcal} - \d_{\Fcal}\log\varepsilon -\frac{c^2\varepsilon^2}{16AV} \Big ),\]
    where $C>0$ depends only on $\Theta, \ell, B, \eta$ and $c$.
 \end{itemize}
 
\end{theorem}

\section{Additional numerical experiments}
\label{a:additional_experiments}
\subsection{Integrating rough functions}
We compare the variance of five quadrature methods for integrating functions on the unit interval. These include vanilla iid sampling (\textsc{iid}), the wavelet-based DPPs \textsc{haar} and \textsc{db2} introduced in this paper (the latter using the adjusted linear statistic $\widetilde{\mathcal{L}}$ in \eqref{eq:adjusted-linear-stat}), the Dirichlet DPP \textsc{dir} considered in \cite{CoMaAm20Sub} and the OPE-based DPP \textsc{ope} built from the first $n$ Legendre polynomials.

We consider four test functions of varying regularity: 
\begin{itemize}
    \item $\varphi_\gamma(x) = \frac{1}{d}\sum\limits_{i = 1}^{d}\frac{f_\gamma(x_i)}{\int_{0}^{1} f_{\gamma}}, \quad f_\gamma(t) = \frac{1}{d}\sum_{i = 1}^{d}|t - 1/2|^\gamma$ for $\gamma \in \{0.25, 0.75\}$,
    \item $\varphi_{\text{mixcos}}(x) = \frac{1}{d} \sum\limits_{i = 1}^{d}\frac{f_\text{mc}(x_i)}{\int_{0}^{1} f_{\text{mc}}}, \quad f_{\rm mc}(t) = 0.1\,|\cos(5\pi(t-1/2))| + (t-1/2)^2$,
    \item $\varphi_{\text{bump}}(x) = \prod\limits_{i = 1}^{d}\frac{f_\text{bump}(x_i)}{\int_{0}^{1} f_{\text{bump}}}, \quad f_{\rm bump}(t) = \exp(-0.1/(t(1-t)))$.
\end{itemize}
We note that the function $\varphi_\gamma$ lies in the Hölder space $C^{0,\gamma}([0,1]^d)$ and in the Sobolev space $H^s((0,1)^d)$ for all $s < \gamma + 1/2$. The function $\varphi_{\text{mixcos}}$, even though non-differentiable, is Lipschitz (i.e $C^{0,1}([0,1]^d)$ and belongs to the space $H^s((0,1)^d)$ for all $s < 3/2$. The last function we consider is $\varphi_{\text{bump}}$, which is infinitely smooth. 

For each function, we generate $100$ independent realizations of every estimator and report the empirical variance against the number of samples $n$ on a log-log scale. The results illustrate that all DPP-based estimators improve over the Monte Carlo rate $n^{-1}$. In $d = 1$, \textsc{haar} and \textsc{db2} offer the most significant variance reduction, comfortably outperforming both \textsc{dir} and \textsc{ope}. In $d = 2$, \textsc{haar} still offers the best performance. On the other hand, while \textsc{db2} still outperforms the non-wavelet DPP samplers for $\varphi_{\gamma}(x)$ for both $\gamma \in \{0.25,0.75\}$, this advantage does not persist for the other two functions.

\begin{figure}[!htbp]
    \centering

    \begin{minipage}{0.48\textwidth}
        \centering
        \includegraphics[width=\linewidth]{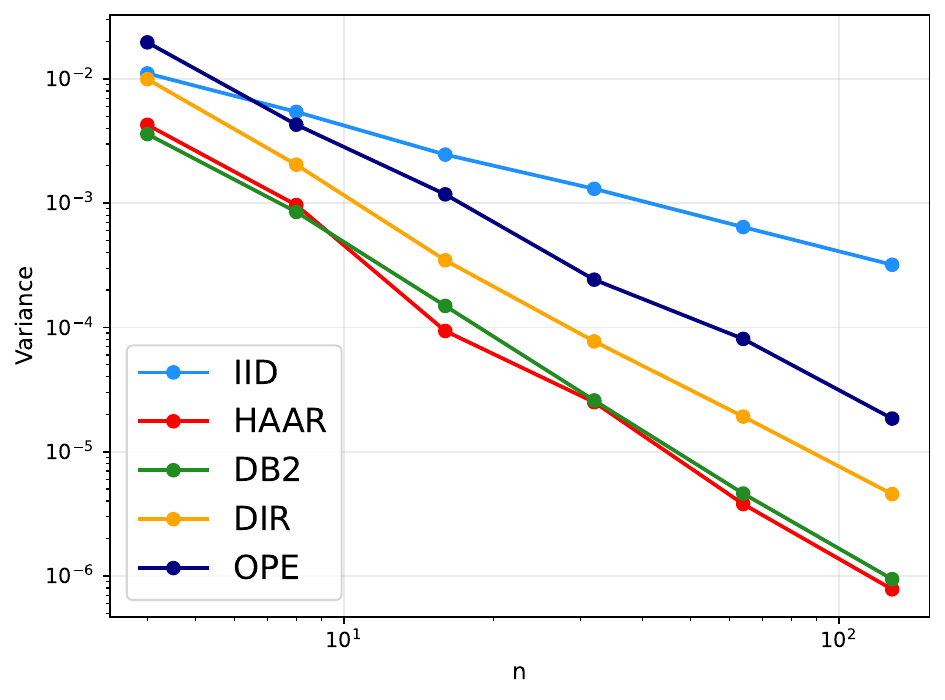}
        \caption{d = 1: $\varphi_{0.25}(x)$}
    \end{minipage}
    \hfill
    \begin{minipage}{0.48\textwidth}
        \centering
        \includegraphics[width=\linewidth]{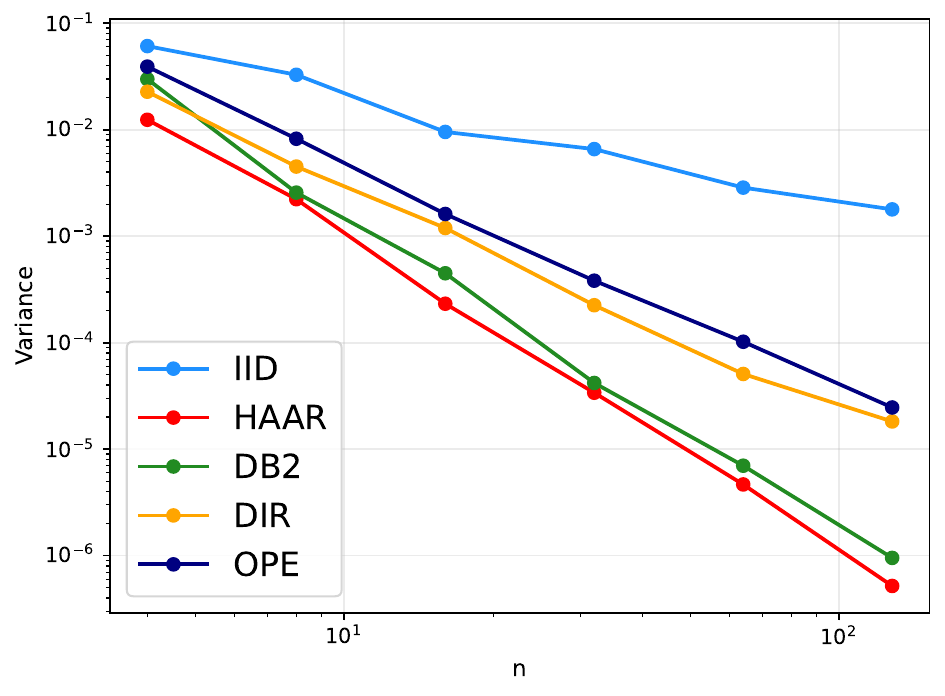}
        \caption{d = 1: $\varphi_{0.75}(x)$}
    \end{minipage}

    \vspace{0.8em}

    \begin{minipage}{0.48\textwidth}
        \centering
        \includegraphics[width=\linewidth]{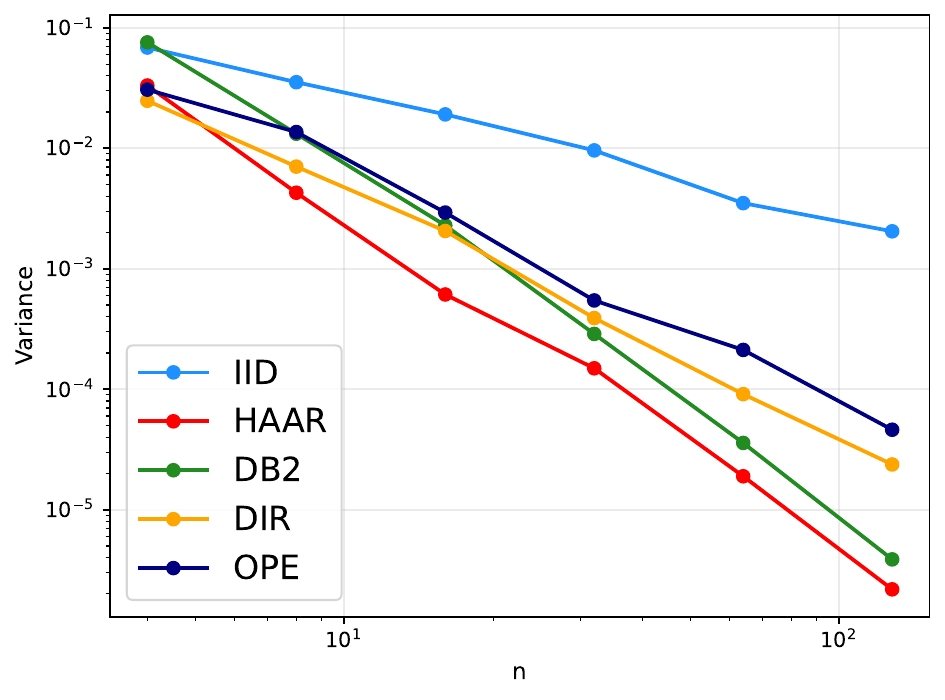}
        \caption{d = 1: $\varphi_{\text{mixcos}}(x)$}
    \end{minipage}
    \hfill
    \begin{minipage}{0.48\textwidth}
        \centering
        \includegraphics[width=\linewidth]{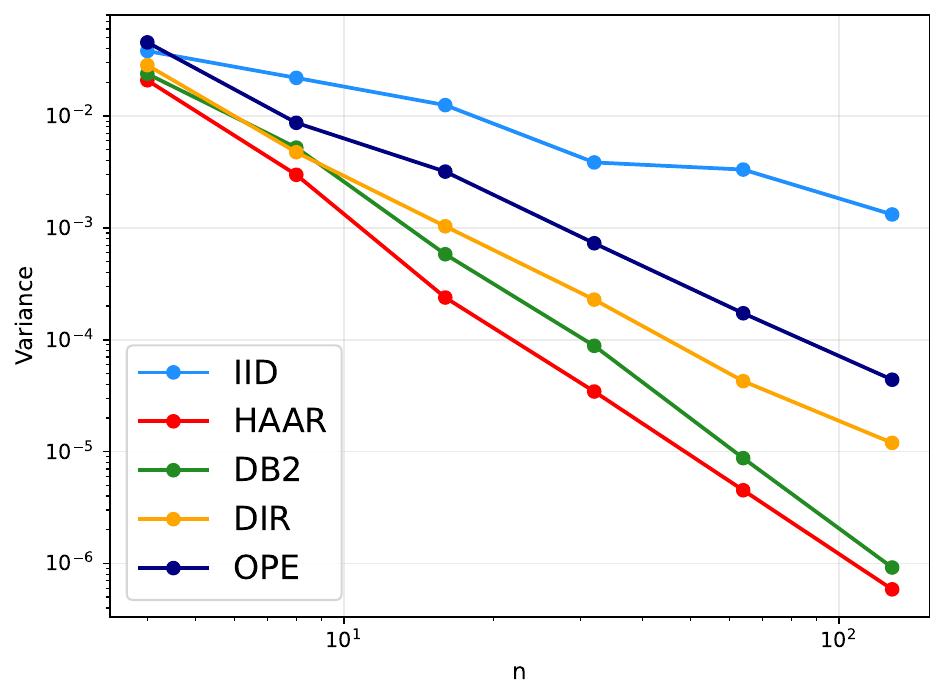}
        \caption{d = 1: $\varphi_{\text{bump}}(x)$}
    \end{minipage}
    \label{fig:variance-decay-four-functions}
\end{figure}

\begin{figure}[!htbp]
    \centering

    \begin{minipage}{0.48\textwidth}
        \centering
        \includegraphics[width=\linewidth]{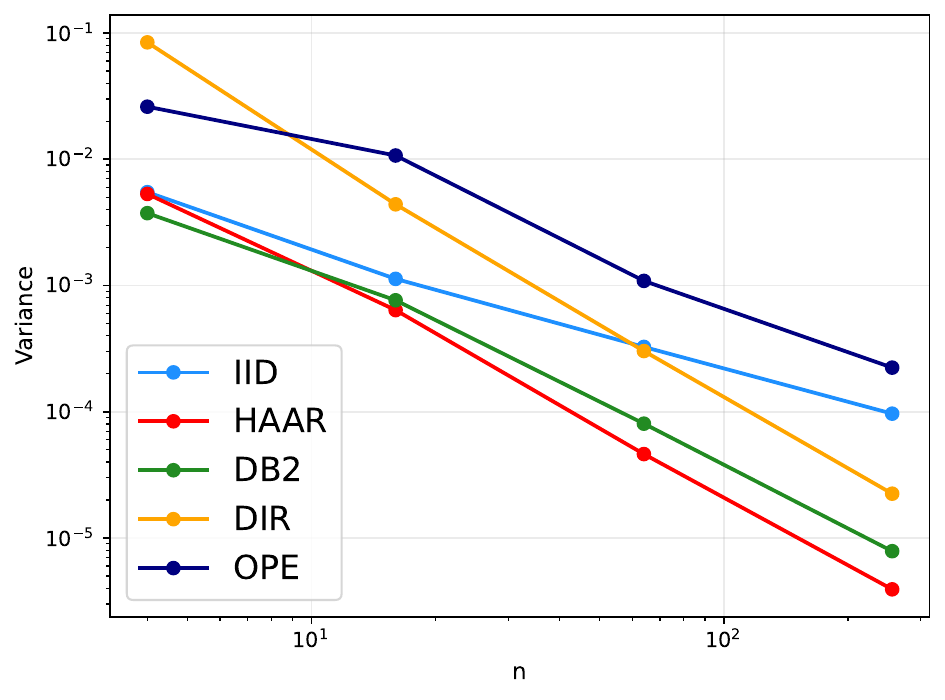}
        \caption{d = 2: $\varphi_{0.25}(x)$}
    \end{minipage}
    \hfill
    \begin{minipage}{0.48\textwidth}
        \centering
        \includegraphics[width=\linewidth]{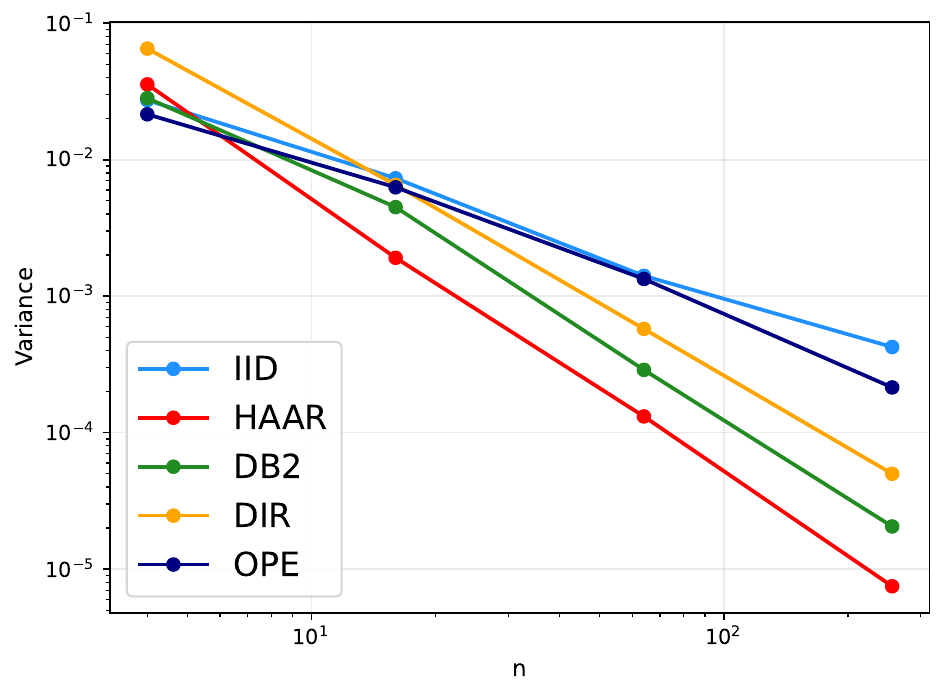}
        \caption{d = 2: $\varphi_{0.75}(x)$}
    \end{minipage}

    \vspace{0.8em}

    \begin{minipage}{0.48\textwidth}
        \centering
        \includegraphics[width=\linewidth]{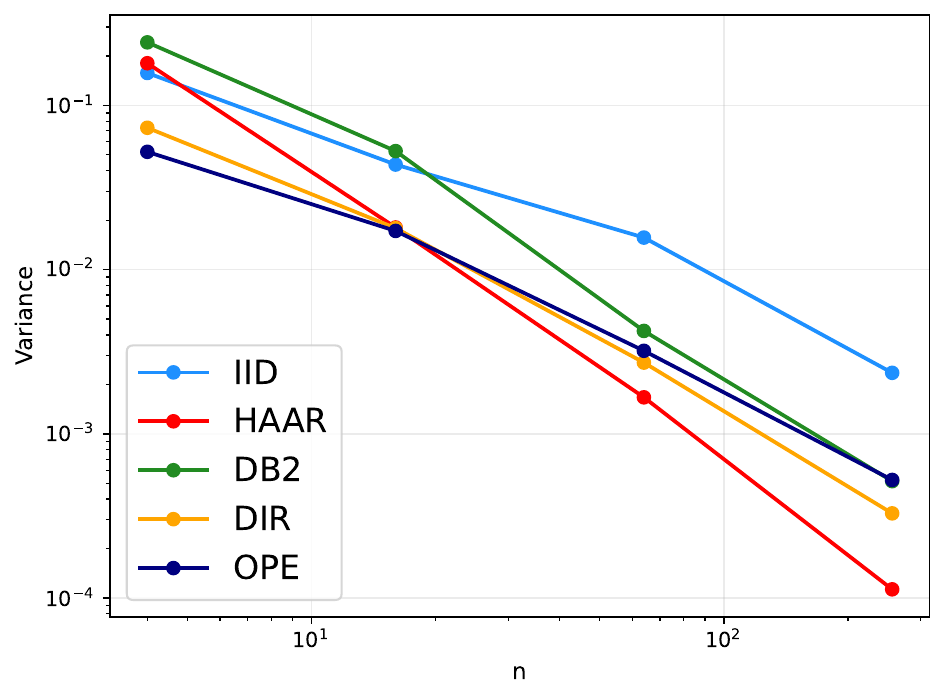}
        \caption{d = 2: $\varphi_{\text{mixcos}}(x)$}
    \end{minipage}
    \hfill
    \begin{minipage}{0.48\textwidth}
        \centering
        \includegraphics[width=\linewidth]{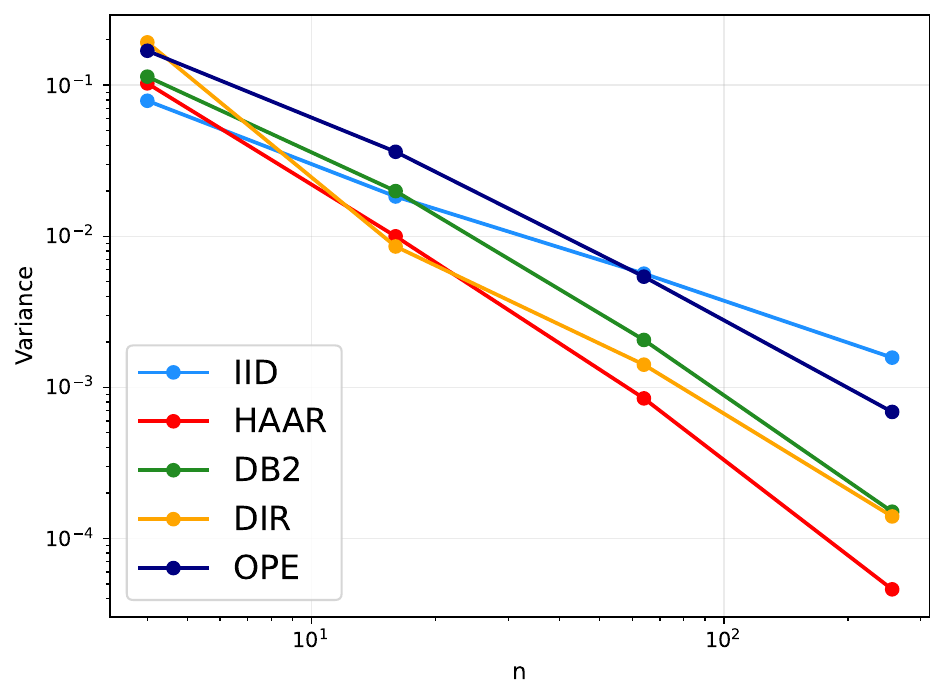}
        \caption{d = 2: $\varphi_{\text{bump}}(x)$}
    \end{minipage}
    \label{fig:variance-decay-four-functions}
\end{figure}

\end{document}